\documentclass[9.5pt,journal,compsoc]{./IEEEtran}
\makeatletter
\long\def\@makecaption#1#2{%
\ifx\@captype\@IEEEtablestring%
\normalsize\bgroup\par\centering\@IEEEtabletopskipstrut{\normalfont\sffamily\footnotesize #1}\\{\normalfont\sffamily\footnotesize #2}\par\addvspace{0.5\baselineskip}\egroup%
\@IEEEtablecaptionsepspace
\else
\@IEEEfigurecaptionsepspace
\setbox\@tempboxa\hbox{\normalfont\sffamily\footnotesize {#1.}\nobreakspace #2}%
\ifdim \wd\@tempboxa >\hsize%
\setbox\@tempboxa\hbox{\normalfont\sffamily\footnotesize {#1.}\nobreakspace}%
\parbox[t]{\hsize}{\normalfont\sffamily\footnotesize \noindent\unhbox\@tempboxa#2}%
\else%
\hbox to\hsize{\normalfont\sffamily\footnotesize\box\@tempboxa\hfil}%
\fi\fi}
\makeatother

\usepackage{cite}
\ifCLASSINFOpdf
\else
\fi
\usepackage[cmex10]{amsmath}
\usepackage{algorithmic}

\ifCLASSOPTIONcompsoc
  \usepackage[caption=false,font=footnotesize,labelfont=sf,textfont=sf]{subfig}
\else
  \usepackage[caption=false,font=footnotesize]{subfig}
\fi
\usepackage{stfloats}
\fnbelowfloat

\usepackage{times}
\usepackage{epsfig}
\usepackage{graphicx}
\usepackage{amsthm}
\usepackage{amssymb}

\usepackage{algorithm,setspace}
\usepackage[letterspace=-80]{microtype}
\usepackage{color, colortbl}

\usepackage{url}

\newtheorem{theorem}{Theorem}
\newtheorem{definition}{Definition}

\definecolor{LightGray}{gray}{0.9}

\begin{document}
%
% paper title
% can use linebreaks \\ within to get better formatting as desired
% Do not put math or special symbols in the title.
\title{Socially Constrained Structural Learning for Groups Detection in Crowd}
%
%
% author names and IEEE memberships
% note positions of commas and nonbreaking spaces ( ~ ) LaTeX will not break
% a structure at a ~ so this keeps an author's name from being broken across
% two lines.
% use \thanks{} to gain access to the first footnote area
% a separate \thanks must be used for each paragraph as LaTeX2e's \thanks
% was not built to handle multiple paragraphs
%

\author{Francesco Solera,
        Simone Calderara,~\IEEEmembership{Member,~IEEE,}
        and~Rita Cucchiara,~\IEEEmembership{Fellow,~IEEE}% <-this % stops a space
\thanks{Authors are with the Department
of Engineering Enzo Ferrari, University of Modena and Reggio Emilia,
Italy e-mail: name.surname@unimore.it}% <-this % stops a space
}

% note the % following the last \IEEEmembership and also \thanks - 
% these prevent an unwanted space from occurring between the last author name
% and the end of the author line. i.e., if you had this:
% 
% \author{....lastname \thanks{...} \thanks{...} }
%                     ^------------^------------^----Do not want these spaces!
%
% a space would be appended to the last name and could cause every name on that
% line to be shifted left slightly. This is one of those "LaTeX things". For
% instance, "\textbf{A} \textbf{B}" will typeset as "A B" not "AB". To get
% "AB" then you have to do: "\textbf{A}\textbf{B}"
% \thanks is no different in this regard, so shield the last } of each \thanks
% that ends a line with a % and do not let a space in before the next \thanks.
% Spaces after \IEEEmembership other than the last one are OK (and needed) as
% you are supposed to have spaces between the names. For what it is worth,
% this is a minor point as most people would not even notice if the said evil
% space somehow managed to creep in.

% The paper headers
\markboth{IEEE TRANSACTIONS ON PATTERN ANALYSIS AND MACHINE INTELLIGENCE}%
{Solera \MakeLowercase{\textit{et al.}}: Socially Constrained Structural Learning for Groups Detection in Crowd}
% The only time the second header will appear is for the odd numbered pages
% after the title page when using the twoside option.
% 
% *** Note that you probably will NOT want to include the author's ***
% *** name in the headers of peer review papers.                   ***
% You can use \ifCLASSOPTIONpeerreview for conditional compilation here if
% you desire.

% If you want to put a publisher's ID mark on the page you can do it like
% this:
%\IEEEpubid{0000--0000/00\$00.00~\copyright~2012 IEEE}
% Remember, if you use this you must call \IEEEpubidadjcol in the second
% column for its text to clear the IEEEpubid mark.

% use for special paper notices
%\IEEEspecialpapernotice{(Invited Paper)}

% As a general rule, do not put math, special symbols or citations
% in the abstract or keywords.
\IEEEcompsoctitleabstractindextext{
\begin{abstract}
Modern crowd theories agree that collective behavior is the result of the underlying interactions among small groups of individuals.
In this work, we propose a novel algorithm for detecting social groups in crowds by means of a Correlation Clustering procedure on people trajectories. The affinity between crowd members is learned through an online formulation of the Structural SVM framework and a set of specifically designed features characterizing both their physical and social identity, inspired by Proxemic theory, Granger causality, DTW and Heat-maps. To adhere to sociological observations, we introduce a loss function ($G$-MITRE) able to deal with the complexity of evaluating group detection performances. We show our algorithm achieves state-of-the-art results when relying on both ground truth trajectories and tracklets previously extracted by available detector/tracker systems.
%Crowds are complex and dynamic entities to capture and analyze. Nevertheless, the events worth to be understood are likely to be limited to the result of a cooperation between members of the same group.
%In this work, we propose a novel algorithm for detecting social groups in crowds by means of a Correlation Clustering procedure. The affinity between crowd members is learned through an online formulation of the Structural SVM framework and a set of specifically designed features characterizing both their physical and social identity. For the solutions to adhere to sociological observations, we introduce a loss function able to deal with the complexity of evaluating group detection performances. We show our algorithm achieves state-of-the-art results both in presence of complete trajectories and when working with a detector/tracker output.
%We tested our algorithms on the proposed datasets as well as on standard benchmarks, showing a strong improvement over other state of the art methods.
\end{abstract}

% Note that keywords are not normally used for peerreview papers.
\begin{IEEEkeywords}
Crowd analysis, group detection, Structural SVM, Correlation Clustering, Proxemic theory, Granger causality.
\end{IEEEkeywords}}

% make the title area
\maketitle
%\tableofcontents

% For peer review papers, you can put extra information on the cover
% page as needed:
% \ifCLASSOPTIONpeerreview
% \begin{center} \bfseries EDICS Category: 3-BBND \end{center}
% \fi
%
% For peerreview papers, this IEEEtran command inserts a page break and
% creates the second title. It will be ignored for other modes.
\IEEEpeerreviewmaketitle

% The very first letter is a 2 line initial drop letter followed
% by the rest of the first word in caps.
% 
% form to use if the first word consists of a single letter:
% \IEEEPARstart{A}{demo} file is ....
% 
% form to use if you need the single drop letter followed by
% normal text (unknown if ever used by IEEE):
% \IEEEPARstart{A}{}demo file is ....
% 
% Some journals put the first two words in caps:
% \IEEEPARstart{T}{his demo} file is ....
% 
% Here we have the typical use of a "T" for an initial drop letter
% and "HIS" in caps to complete the first word.
%\vspace{.5cm}
%\begin{verse}
%\textit{Now it was just the two of us, Max and me. And about a thousand other people around us.}\\
%\vspace{.1cm}
%\hfill The Pledge
%\end{verse}

\section{Introduction}
\label{sec:intro}
\IEEEPARstart{C}{rowd} phenomena are complex and their logic still escapes formal rules and precise social explanations.
%Crowd phenomena are complex and their logic still escapes formal rules while contemporarily exposing fascinating challenges.
% in a variety of research fields like collective psychology, system theory, sociology and recently computer vision.
Eventually, the ambition of crowd analysis is to characterize people behaviors, predict and prevent potentially dangerous situations and improve the well-being of communities.
This has been traditionally provided by simulation models~\cite{vizzari12} or automatic video analysis~\cite{ge_vision-based_2012}.
%Eventually, the ambition is always to precisely characterize people behavior in crowds, to predict and prevent potentially dangerous situations by means of either synthetic simulation models or automatic visual analysis.
%
%\emph{Crowds are crowds of groups} is one of the emerging theory~\cite{baqndini13,moussaid_walking_2010,mcphail_using_1982},
Recently, \emph{groups} have been recognized as the 
%Before understanding how people behave in crowds, from both a collective and an individual perspective, we as scientists need to understand and agree on what a crowd is.
%\emph{Crowds are crowds of groups} is the recent statement of many social theories and recent empirical observation~\cite{baqndini13,moussaid_walking_2010,mcphail_using_1982}.
%Groups are recognized as the
basic elements which compose the crowd~\cite{moussaid_walking_2010},
leading to an intermediate level of abstraction that is placed between two outfacing views: the crowd as a flow of indistinguishable people~\cite{Moore:2011:VCS:2043174.2043192} and its interpretation as a collection of individuals~\cite{PhysRevE.51.4282}.
Identifying groups is consequently a mandatory step to grasp the complex social dynamics ruling collective behaviors in crowds.
This poses new challenges for computer vision, since groups are definitely more difficult to characterize than pedestrians acting alone or as a whole.
% detection in video streams is definitely less studied than pedestrian analysis or crowd flow motion estimation.
%Thanks to the large amount of surveillance data available all around the world and the rapid growth of computational power, our field has recently witnessed an extensive use of trajectories for human behavior understanding. %at different levels
% Very basic activity recognition can be accomplished by observing the motion patterns of one or few individuals\ref{}, while by analyzing the spatio-temporal flow of a large crowd it is possible to detect unexpected events or anomalies in the way someone is interacting with the rest of the people.
%
%\begin{figure}[tbh]
%\centering
%	\begin{subfigure}[b]{0.49\columnwidth}
%		\includegraphics[width=\textwidth]{images/university1.png}
%	\end{subfigure}%
%	~
%	\begin{subfigure}[b]{0.49\columnwidth}
%		\includegraphics[width=\textwidth]{images/1shatian3.png}
%	\end{subfigure}
%	\caption{Examples of social groups detected in crowds.} 
%	\label{fig:crowd}
%\end{figure}
%We therefore propose a method for learning to visually detect groups in low/medium density crowds, as shown in Fig. \ref{fig:crowd}. Our work is built upon the joint adoption of sociologically grounded features and a learning framework that can specialize the concept of group depending on the peculiarities of both the scenario and the crowd itself.
%One of our claim is in fact the impossibility of giving a unique computational definition of what a group is. We will follow the sociological interpretation of groups \cite{}, describing groups as
%

In this work, we propose a learning based solution for visually detecting groups in low/medium density crowds (Fig.~\ref{fig:crowd}) under the hypothesis that the \emph{concept of group} can be visually discerned and people trajectories can be extracted up to some extent. The strong novelty of our approach is the joint adoption of sociologically grounded features and a learning framework able to specialize the concept of group accounting for different scenarios, motion constraints and crowd densities.
%We claim in fact it is not possible to give a unique computational definition of what a group is. Nevertheless, we will adhere to the following sociological interpretation~\cite{Tajfel01041974,turner_significance_1986}.
To this end, we adhere to a classical sociological interpretation of groups~\cite{turner81}, which can be formalized as follows.
\begin{definition}
\label{def:group}
A group is defined as two or more people interacting to reach a common goal and perceiving a shared membership, based on both physical and social identity.
\end{definition}
Accordingly, we propose a new formulation of the problem of detecting groups in crowds as a supervised \emph{Correlation Clustering} (CC)~\cite{bansal_correlation_2002}. We solve it through a \emph{Structural Support Vector Machines} (Structural SVM)~\cite{joachims_cutting-plane_2009} framework that learns a context dependent distance measure, based on a set of features inspired by Def.~\ref{def:group}  effective on both ground truth trajectories and automatically obtained tracklets. The design of socially grounded features is one of the main contributions of the work.

\begin{figure}[t!]
\centering
	\subfloat[\texttt{student003}]{
		\includegraphics[width=0.31\columnwidth]{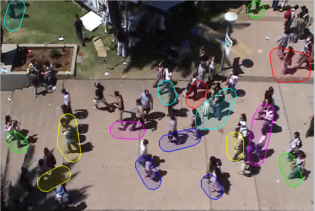}
	}
	\subfloat[\texttt{1shatian3}]{
		\includegraphics[width=0.31\columnwidth]{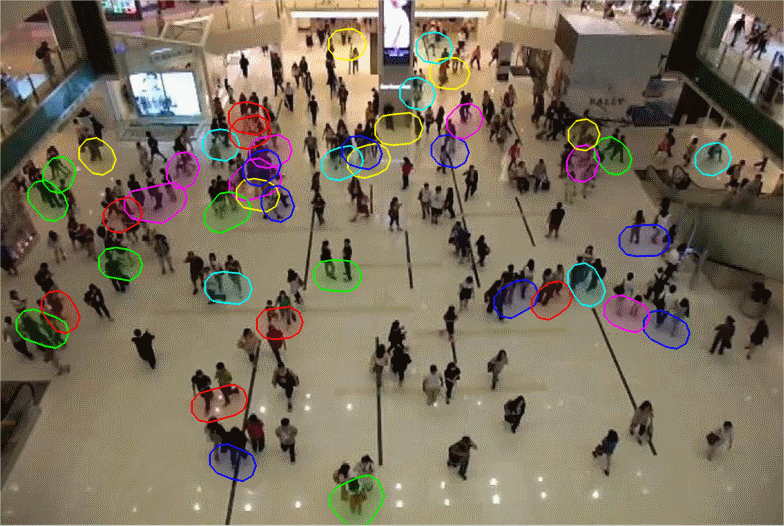}
	}
	\subfloat[\texttt{1dawei1}]{
		\includegraphics[width=0.31\columnwidth]{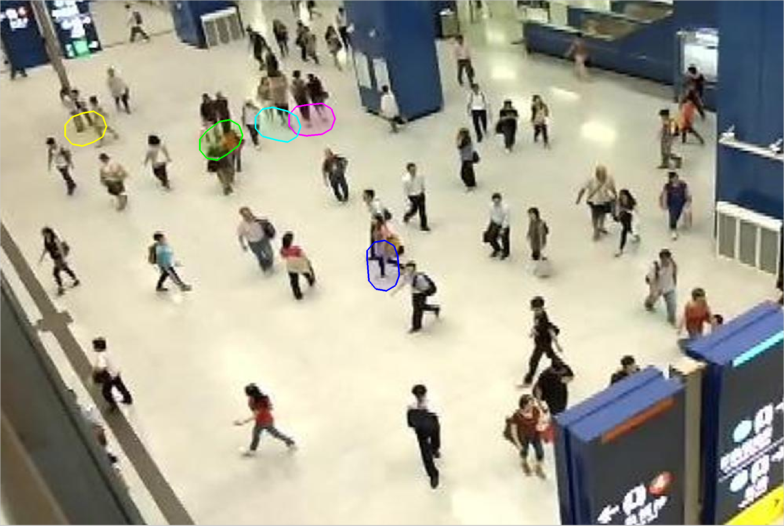}
	}
	\caption{Examples of social groups detected in crowds.} 
	\label{fig:crowd}
\end{figure}  

Moreover, a new socially based \emph{loss function} ($G$-MITRE) is defined for the Structural SVM.
Differently from previous solutions~\cite{pellegrini_improving_2010} and~\cite{ge_vision-based_2012}, our approach doesn't rely on scene-dependent parameters that would limit the applicability of the method in real world contexts. Finally, we also propose an online learning procedure that handles smooth variation in crowd composition and density, useful in online surveillance.

We annotated and made publicly available two new datasets: \emph{MPT-${\bf 20}$x${\bf 100}$} and \emph{GVEII} (see Sec.~\ref{sec:exp}). Results on standard benchmarks, as well as on the proposed datasets, outperform current methods.
We strongly believe that an automatic system for group detection will influence future public area visual surveillance and will bring benefits to modeling and simulation application for architectural planning by providing real and precise data observation of crowds phenomena.

\section{Related Work}
The modeling of pedestrian dynamics in crowds represents a relatively recent research field. Most of the works are based on sociological paradigms
%developed in the past decades
and computer vision based approaches have also evolved under the influence of these theories.

\subsubsection*{Modeling and Observing the Crowd}
Most of the research work has tried to tackle the crowd as an exclusively collective phenomenon, where individuality does not exist. This recalls the primitive \emph{Popular Mind Theory}~\cite{bon2003crowd} by Gustave Le Bon, where the crowd was defined as a ``pathological monster with no individual consciousness''. Accordingly, crowds have been analyzed by means of physical models (\emph{e.g.} hydrodynamics~\cite{Moore:2011:VCS:2043174.2043192}), neglecting the existence of single individual purposes and goals. However, these models are effective mainly in extremely dense crowds.
Conversely, many other approaches have been inspired  by the 70s \emph{Social Loafing Theory}~\cite{Ingham1974371}, which stated that individuality was a strong requirement for the pursuit of personal goals. Helbings \emph{Social Force Model}~\cite{PhysRevE.51.4282}, which asserts that anyone movements towards her goals are influenced by the surrounding pedestrians, has been the main building block for many crowd modeling and analysis works, ranging from abnormal behavior detection~\cite{5206641} to tracking~\cite{5509779}.
Recently, studies on people attending events have underlined that most of the people tend to move in groups and social relations influence the way people behave in crowds~\cite{moussaid_walking_2010,bandini_crowd_2012}.
These empirical observations are supported by Reicher in the recent \emph{Social Identity Model of Deindividuation Effects}~\cite{Reicher95}, which assumes that crowd behavior is regulated by the social rules and behaviors groups choose to adopt. This is the main social paradigm underpinning our research too.

\subsubsection*{Visual Detection of Groups in Crowds}
It was only recently that group detection showed promising results. The process is in fact built upon several open challenges in computer vision, starting from people detection and tracking in crowds~\cite{Rodriguez11} to analyzing and grouping extracted trajectories~\cite{solera_structured_2013}.

Some works employ the concept of \emph{F-formations} by Kendon~\cite{kendon1990conducting} to discern group formation process. Broadly speaking, F-formations can be seen as specific positional and orientational patterns that people must sustain in order to be considered engaged in a social relationship. Despite robust results~\cite{cristani2011social}, this theory is suited to stationary groups only and is not defined for moving groups, a case which cannot be ignored in crowd analysis.

Thus, complementary approaches analyze pedestrians motion paths; according to the type of available tracklets, they can be partitioned in group-based, individual-group joint and individual-based.
%A significant amount of work focuses on the performance improvements that can be achieved by tracking algorithms when considering groups as structured entities of the scene \cite{Pang08}.
%This is the case of \textit{group-based} approaches, where groups are considered as atomic entities in the scene as no higher level information can be extracted neatly, typically due to high noise or high complexity of crowded scenes \cite{Wang06, Feldmann11, Lin07}. Since these models are too simplistic to be used to further infer on groups behavior, \textit{individual-group joint} approaches try to overcome the lack of finer information by hypothesizing trajectories while tracking groups at a coarser level \cite{Pang08, Bazzani12}.
In \textit{group-based} approaches, groups are considered as atomic entities in the scene since no higher level information can be extracted neatly, typically due to high noise or high complexity of crowded scenes \cite{Feldmann11, shao14}. Since these models are often too simplistic to further infer on groups behavior, \textit{individual-group joint} approaches try to overcome the lack of finer information by hypothesizing trajectories while tracking groups at a coarser level \cite{Pang08, Bazzani12}.
Finally, \textit{individual-based} tracking algorithms build up on single pedestrians trajectories.
%, which are the most informative features we can hope to extract in a crowded scene.
This kind of approach has been gaining momentum only recently since tracking even in high density crowds is becoming everyday a more feasible task~\cite{Rodriguez11}.
Pellegrini~\emph{et al.}~\cite{pellegrini_improving_2010} employ a Conditional Random Field to jointly predict trajectories and estimate group memberships, modeled as latent variables, over a short time window.
Yamaguchi~\emph{et al.}~\cite{yamaguchi_who_2011} predict whether two pedestrians are in the same group through a linear SVM on trivial distance, speed difference and time overlap information.
%Yamaguchi~\emph{et al.}~\cite{yamaguchi_who_2011} predict groups as minimization of an energy function that encodes physical condition, personal motivation and social interactions features.
%Such a formulation loses those properties granted by the use of a graphical model such as transitivity; as a consequence groups will be a covering and not a partition of the pedestrians set.
Recently, Chang~\emph{et al.}~\cite{Chang11} proposed a soft segmentation process to partition the crowd by constructing a weighted graph, where the edges represent the probability of individuals to belong to the same group. 
An interesting unsupervised approach is Zanotto~\emph{et. al} ~\cite{zanotto12}, where a potentially infinite mixture model is fitted on pedestrians, regarded as sampled observations from the mixture. Previous frames data and predictions are used as prior information for the models (one for each group), but pairwise relations between individuals are neglected as groups are modeled only through the mean position and velocity of their members.
%An interesting unsupervised approach is Zanotto et. al~\cite{zanotto12}, where groups are classified through a set of infinite mixture distribution modeling proxemic-inspired features, even if their analysis is frame-by-frame and the dynamics of trajectories are not taken into account.
%
Above all, we mention Ge~\emph{et al.}~\cite{ge_vision-based_2012} that suggests the use of an agglomerative approach to cluster trajectories, as we do.
%similar to the one proposed in our previous published work~\cite{solera_structured_2013} and further extended in this proposal.
%as we share an agglomerative approach to associate trajectories.
They hierarchically merge clusters by evaluating a well-founded sociological inter-group closeness measure defined on a combination of proximity and velocity features, stopping when a given condition is met.
%The work in \cite{ge_vision-based_2012} is based on well-founded sociological priors, but the use of fixed thresholds is often a limit when coping with complex crowds and groups dynamics.
%tries to follow sociological studies \cite{hall66,mcphail_using_1982} by developing a test based on fixed thresholds to verify group membership, this methodology turns out to be insufficient for a robust group detection system becuase of the dynamic meaning of groups.
%All these works actually present some drawbacks, ranging from lack of sociological motivation in the choice of features~\cite{pellegrini_improving_2010} to the na\"{i}ve understanding that groups can be considered as such even without modeling sociological aspects of transitivity~\cite{Chang11} or by ignoring it at all~\cite{yamaguchi_who_2011}. The use of fixed thresholds and parameters as in~\cite{ge_vision-based_2012} is also a limiting aspect when dealing with dynamic concepts as groups are.

Conversely, our method does not rely neither on relative position or velocity fixed thresholds~\cite{ge_vision-based_2012,zanotto12} nor on sequence-dependent parameters~\cite{pellegrini_improving_2010}; it is flexible and general as the features are not scene-specific~\cite{Chang11} and their contribution is learned from examples. Thanks to the use of a clustering inference rule, solutions proposed by our method are partitions and not coverings of the members of the crowd~\cite{yamaguchi_who_2011}, meaning that pairwise relations are consistent with the overall group structure found. Moreover, the use of a time window to predict groups let the method recognize that non-trivial behaviors (\emph{e.g.}~neglecting strict proximity) may occur, whereas frame-by-frame methods are limited to short term reasoning~\cite{zanotto12}. Yet, the discriminative nature of the employed framework makes learning compelling in terms of both required data and computational cost, as opposed to graphical models optimizing over a multiple hypothesis space~\cite{pellegrini_improving_2010}.\\

\noindent This work extends our preliminary attempt in~\cite{solera_structured_2013}. Here we prove our proposal complies with social theories of group formation, we devise and investigate new features to better adhere to the sociological theory underpinning our method and, eventually, extend the tests to new remarkably complex datasets and compare with more recent competing algorithms.
Besides, the experiments further probe the need for learning when dealing with heterogeneous crowds, shedding light on the nature of the problem itself.
%For these reasons, in this proposal we start with pedestrian trajectories, possibly being the output of a detector/tracker, and describe a parameter-less method able to adjust the importance given to different features, in order to grasp groups interactions independently of the context.

\section{Problem Definition}
\label{sec:problem_def}
%We cast the task of finding groups inside a crowded scene as the problem of learning to partition the members of the crowd distinguishing groups from single pedestrian (singletons hereinafter).\\
%As a partitioning problem the group detection task can be formulated as a 
%
We cast the group detection task as a clustering problem. Consider a set of pedestrian $M = \{a, b, \dots\}$ and $\mathcal{Y}(M)$ as the set of all possible ways to partition $M$. Defining $y$ as a subset of pedestrians (also referred to as group or cluster) in $M$, a generic set of subsets ${\bf y}~=~\{y_1, y_2, \dots\}$ is a valid solution in $\mathcal{Y}(M)$ if the partitioning axioms are satisfied: $\forall a\in\mathcal{M}, \exists!y\in\mathcal{Y}(\mathcal{M}):a\in y$ and $\cup_{y\in\mathcal{Y}(\mathcal{M})}y = \mathcal{M}$.
%\emph{i)} no pedestrian must be found in more than one cluster and \emph{ii)} the union of all the clusters in ${\bf y}$ must be a covering for $M$.
%There are two trivial partitioning solutions: one where all the pedestrians are clustered as a single group ($|{\bf y}|=1$) and one where each pedestrian belongs to different clusters ($|{\bf y}|=|M|$).
%For the rest of the paper we call \emph{singletons} those pedestrians whose cluster is composed by themselves only, \emph{i.e.} $|y| = 1$.
Here, we call \emph{singletons} those pedestrians whose cluster is composed by themselves only, \emph{i.e.} $|y| = 1$.

%\begin{figure}[tbh]
%\centering
%\includegraphics[width=\columnwidth]{images/groups}
%\caption{An example of a legal partitioning of a set of elements into clusters. Singletons are individuated by a cross mark, while groups are identified by their color.}
%\label{fig:clusters}
%\end{figure}

In crowded contexts, this grouping cannot be solved by exploiting spatial (positional or orientational) information only, as proposed in F-formation theory, due both to confusion and motion. Moreover, it is often the case that the physical distance between a singleton and a member of a cluster is lower than that cluster intra-member mean distance. This is due to the fact that, in real situations, social aspects heavily intervene in the group formation process.
%
%
% and, as a consequence, it is not possible to set a threshold to discern groups according to the distance only.
%Complexities in problem solution arise whenever the distance between a singleton and a member of a cluster is lower than that cluster intra-member mean distance. This is due to the fact that, in real situations, social aspects intervene in the group formation process and, as a consequence, it is not possible to set a threshold to discern groups according to the distance only.
%A good solution should also account for transitive relationships: two members may be connected by means of a third member closely related to both of them.
%
In order to obtain crowd partitions that are meaningful from a sociological point of view, 
the following relevant properties of social groups
%, according to the most accounted sociological theories,
must hold.

{\bf Hierarchical Coherence.}
Groups are composed by individuals and sub-groups in a recursive fashion (Fig.\ref{fig:propertya}). This has been first observed in the seminal work of Canetti~\cite{canetti_crowds_1984}, based on the assumption that members within a group cannot erase already settled relationships as the crowd assembles.

%While theorizing the hierarchical architecture of crowds, Canetti \cite{canetti_crowds_1984} coined the term \emph{crystals} to identify the small structured elements composing them. As the crowd is composed of groups, we also require groups to be (eventually) composed of smaller subgroups, but we deny the possibility to members to erase already settled relationships as the crowd assembles, as shown in Fig.~\ref{fig:propertya}.\\
%While theorizing the hierarchical architecture of crowds, Canetti \cite{canetti_crowds_1984} coined the term \emph{crystals} to identify the small structured elements composing them. As the crowd is composed of groups, we also require groups to be (eventually) composed of smaller subgroups, but we deny the possibility to members to erase already settled relationships as the crowd assembles, as shown in Fig.~\ref{fig:propertya}.\\
%
%whenever two separate groups merge as a result of the shrinking of their members' mutual distances they produce a new group and no member of the previous groups is left apart nor outer members are added, or alternatively their previous existence was a contradiction.
%
{\bf Density Invariance.} To keep their group identities preserved at different crowd densities, members must be willing to change the inner distance among them. Groups in very crowded scenes will be more closed and compact, while groups in open spaces will tend to exhibit more dilated patterns (Fig.~\ref{fig:propertyb}); sociologically and empirical evidence can be found in Bandini \emph{et al.}~\cite{bandini_crowd_2012} and in Moussaid~\emph{et al.}~\cite{moussaid_walking_2010}.

{\bf Transitivity.} Not every member of a group needs to be strictly connected with every one else, but any two members may be part of the same group by means of a sufficiently dense subgroup of pedestrian standing between them (Fig.~\ref{fig:propertyc}). McPhail and Wohlstein's work \cite{mcphail_using_1982} formalized this idea: to be considered part of a group one typically will have to be connected with at least half of the members.
%is any number of strictly connected members between them such that the net effect of all the involved pairwise relations is still positive. 

\begin{figure}[t!]
	\centering
        \subfloat[]{
                \includegraphics[width=0.3\columnwidth]{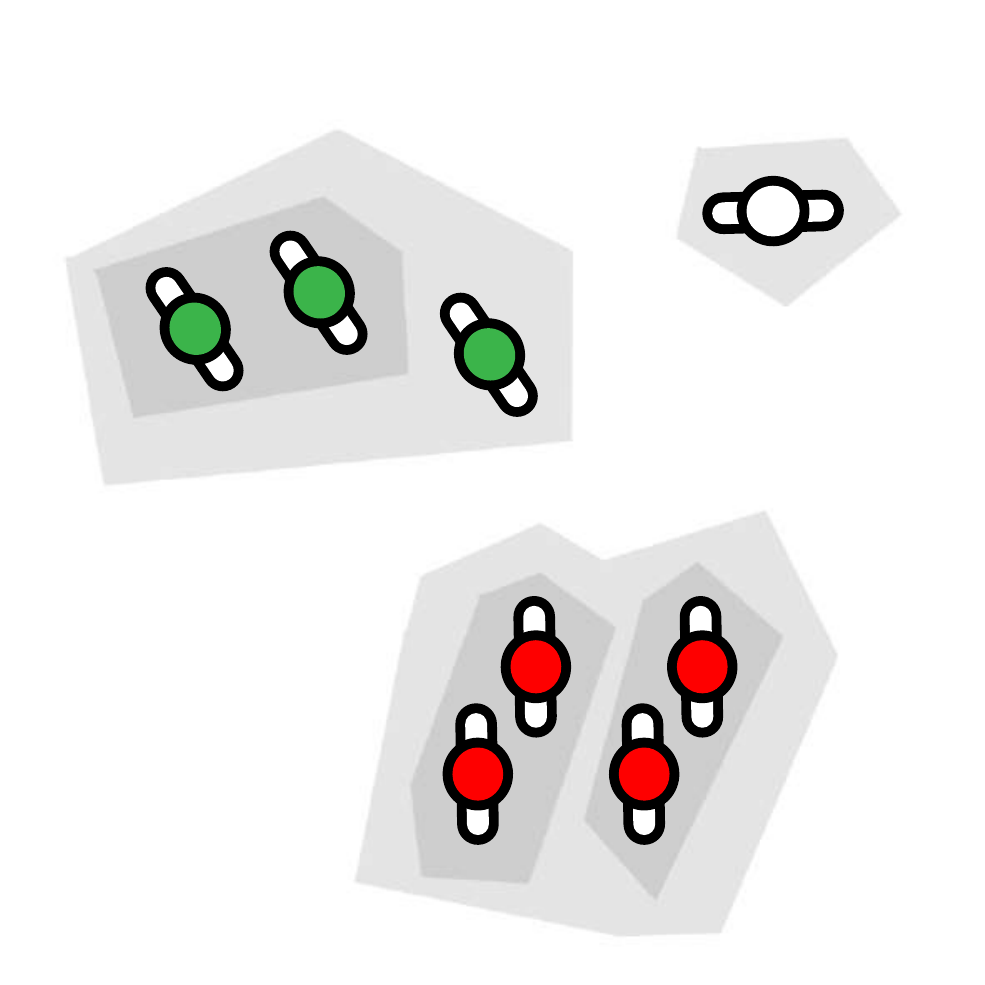}
                \label{fig:propertya}
        }~
		\subfloat[]{
                \includegraphics[width=0.3\columnwidth]{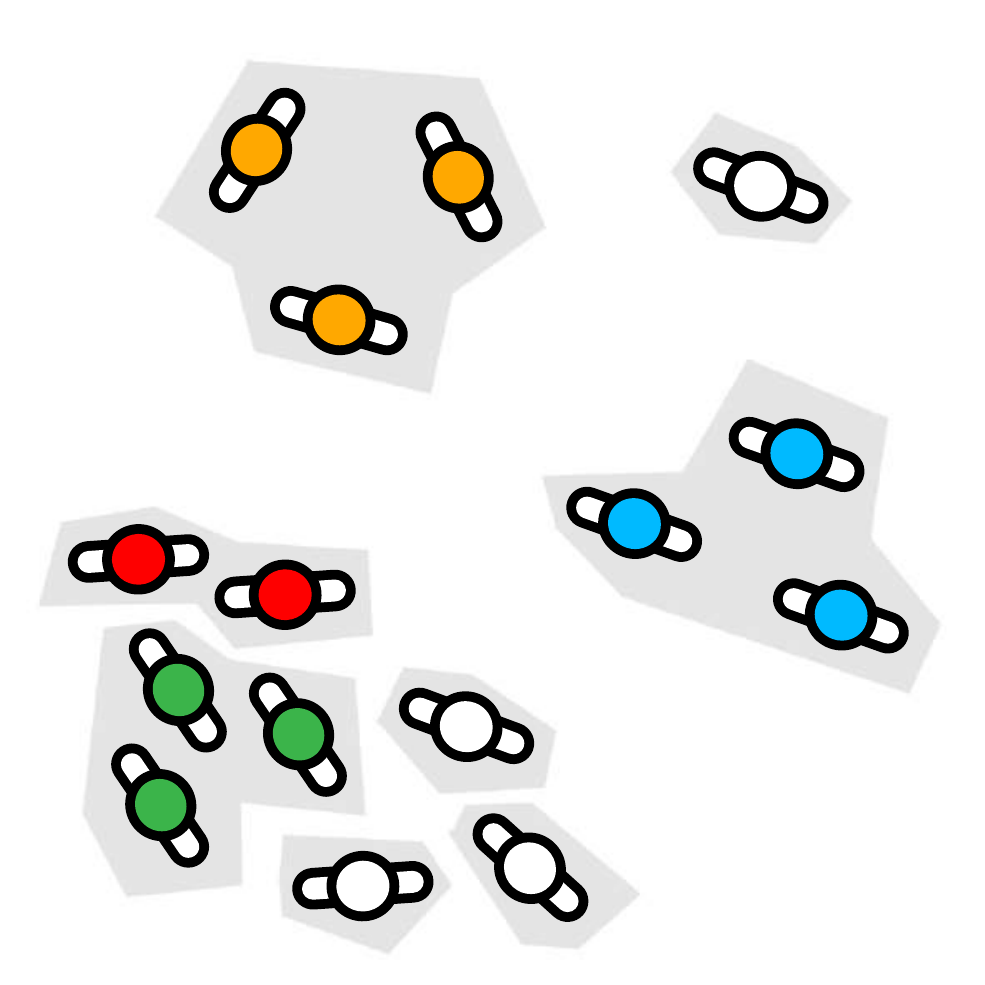}
                \label{fig:propertyb}
        }~
        \subfloat[] {
                \includegraphics[width=0.3\columnwidth]{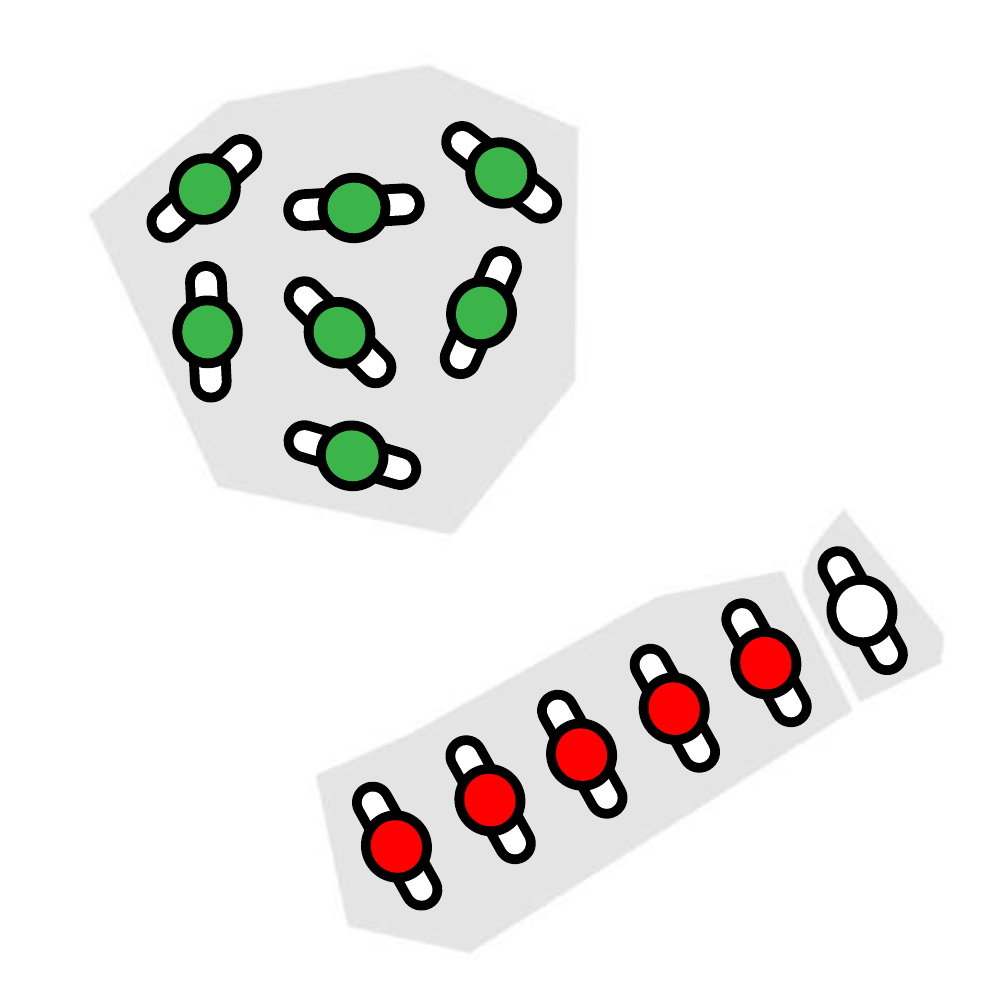}
                \label{fig:propertyc}                
        }
\caption{Highlights of social groups properties: (a) \emph{hierarchical coherence}, (b) \emph{density invariance} and (c) \emph{transitivity}.}  
%\caption{Properties of social groups: In (a) large groups are formed by smaller sub-groups according to the \textit{hierarchical coherence} axiom. In (b) distances among group members scales as a function of the crowd density for the \textit{density invariance} property. In (c) first and last members of the red group are together due to the strong intermediate relations among the others members, while the white singleton isn't sufficiently connected to the group to be part of it, as suggested by the \textit{transitivity property}.}          
\end{figure}

\section{Socially Constrained Clustering for Groups Detection}
\label{sec:solution}
We propose to solve the crowd partitioning problem employing the \emph{Correlation Clustering} (CC)~\cite{bansal_correlation_2002} and we prove it is possible to achieve a quasi-optimal crowd partition guaranteed to satisfy the three aforementioned properties of Sec.~\ref{sec:problem_def}.
The CC algorithm takes as input an affinity matrix $W$ where, if $W^{ab}>0$ ($W^{ab}<0$), elements $a$ and $b$ belong to the same (different) cluster with certainty $|W^{ab}|$. The algorithm returns the partition ${\bf y}$ of a set of elements $M=\{a, b, \dots\}$  so that the sum of the affinities between item pairs in the same clusters $y$ is maximized:
\begin{equation}
\label{eq:correlation_clustering_objective}
\text{CC} = \arg\max
_{{\bf y}\in\mathcal{Y}(M)}\sum_{y\in{\bf y}}\sum_{a\neq b\in y}W^{ab}_{\bf d}.
\end{equation}
The pairwise elements affinity in $W$ is parameterized as weighted linear combination of a bounded dissimilarity measure and its complement:
\begin{equation}
\label{eq:cc_affinity_parametrization}
W^{ab}_{\bf d} = {\boldsymbol\alpha}^T ({\bf 1} - {\bf d}(a, b)) - {\boldsymbol\beta}^T {\bf d}(a, b).
\end{equation}
To be consistent
% that contribute to the group formation process and recalling
with the definition of groups of Sec.~\ref{sec:intro}, we devise the pairwise distance between pedestrian $a$ and $b$, ${\bf d}(a, b)$
%, as a vector of pairwise distances built upon different aspects that concur to unveil the presence of groups,
as detailed in Sec.~\ref{sec:features}.
%In detail, we measure the physical relation between pedestrian $d_{ph}$, their mutual influences in motion pattern by causality and trajectories shape analysis, $d_{ca}$ and $d_{sh}$, and their simultaneous convergence to peculiar zones in the scene $d_{he}$ obtaining:
%\begin{equation}
%{\bf d}(a,b) = [d_\text{ph}, d_\text{sh}, d_\text{ca}, d_\text{he}]   
%\end{equation}

In clustering theory, changing the dissimilarity space results in different partitioning of the domain through the same algorithm.
%and, if every possible configuration can be reached than the richness property holds for that algorithm~\cite{kleinberg_impossibility_2002}.
%Since the complete topology of the clustering solution space cannot be explored being more than exponential in the number of elements, approximate solutions need to be taken into account. Nevertheless,
By tuning $[{\boldsymbol\alpha}, {\boldsymbol\beta}]$ parameters in Eq.~\eqref{eq:cc_affinity_parametrization} we can evaluate many different groupings and we'll show that, under a restrict set of hypothesis, they all satisfy the social properties previously mentioned.
In order to efficiently learn those parameters according to different peculiarities groups exhibit in different scenarios, in Sec.~\ref{sec:learning} we introduce Structural SVM~\cite{tsochantaridis_large_2005} with both an approximated inference procedure and a loss function specifically designed for accurately measuring the compatibility among possible crowd partitions.
\\
The solution to Eq.~\eqref{eq:correlation_clustering_objective}, given the parametrization introduced in Eq.~\eqref{eq:cc_affinity_parametrization} and subject to a hierarchical inference procedure, guarantees the satisfaction of all the social groups properties:
%The peculiar optimization objective of CC in Eq.~\eqref{eq:correlation_clustering_objective}, the respective parametrization introduced in Eq.~\eqref{eq:cc_affinity_parametrization} and a hierarchical inference procedure are sufficient to guarantee that the resulting partition will satisfy all the 3 properties of the social group formation process, as in the following theorem.
\begin{theorem}
When the pairwise elements affinity in $W$ is a weighted linear combination of a bounded similarity measure and its complement, a bottom-up approximated solution to CC produces a partition that respects the hierarchical coherence, density invariance and transitivity properties of social groups.
\end{theorem}

\begin{proof}
Let ${\bf d}:M\times M\rightarrow[0, 1]^p$ be a bounded distance on the set of members of a crowd $M$ so that $(M, {\bf d})$ is a dissimilarity space and suppose the affinity matrix of CC is constructed as in Eq.~\eqref{eq:cc_affinity_parametrization}, for some appropriate positive values of ${\boldsymbol\alpha}, {\boldsymbol\beta} \in \mathbb{R}^p$. To demonstrate that the \emph{density invariance} holds for all solutions of CC consider that when the density increases, both distances between groups and between members of the same group diminish. This phenomenon is a less formal statement of the scale invariance axiom of clustering defined in Kleiberg \cite{kleinberg_impossibility_2002} which is known to hold for sum-of-pairs clustering algorithm.
We must thus show that it holds when we are maximizing affinities instead of minimizing distances as well. To this aim let ${\bf d} = \lambda\bar{\bf d}$ and $\bar{\bf d}:M\times M\rightarrow[0, \frac{1}{\lambda}]^p$ so that
\begin{equation}
\begin{aligned}
        W_{\bf d} &= {\boldsymbol\alpha}^T ({\bf 1} - \lambda\bar{\bf d}) - {\boldsymbol\beta}^T \lambda\bar{\bf d}\\
        		  &= \lambda[{\boldsymbol\alpha}^T(\frac{1}{\lambda} - \bar{\bf d})-{\boldsymbol\beta}^T\bar{\bf d}] = \lambda W_{\bar{\bf d}},
\end{aligned}
\end{equation}
% = kW_{\bar{d}},
where the notation for the elements is dropped for clarity. Consequently, CC satisfies the scale invariance axiom since multiplying all distances by a constant results in multiplying the total affinity of each cluster by a constant and hence the maximum affinity clustering solution is not changed. \emph{Transitivity} follows directly from the objective function of CC in Eq.~\eqref{eq:correlation_clustering_objective}:
%while some elements of $M$ may be considered distant, and thus be described by a negative affinity,
to be assigned to the same group it suffices the existence of any number of members such that the net effect of all the involved pairwise relations is non-decreasing.
Last, the \emph{hierarchical coherence} requires a greedy approximation algorithm to optimize the CC that initially consider each pedestrian in its own cluster and then iteratively merges the two clusters whose union would produce the best clustering score, stopping when joining clusters would decrease the overall affinity.
Hence, elements in the same cluster at lower levels of the hierarchy are also together in higher level clusters.
\end{proof}

%\noindent The algorithm stops when joining clusters would decrease the overall affinity of the solution and the worst case complexity is now reduced to $\mathcal{O}(n^3)$, as opposed to the NP-hardness~\cite{bansal_correlation_2002} of the original problem.

\begin{figure*}[t!]
\centering
        \subfloat[Physical distance]{
                \includegraphics[width=0.22\textwidth]{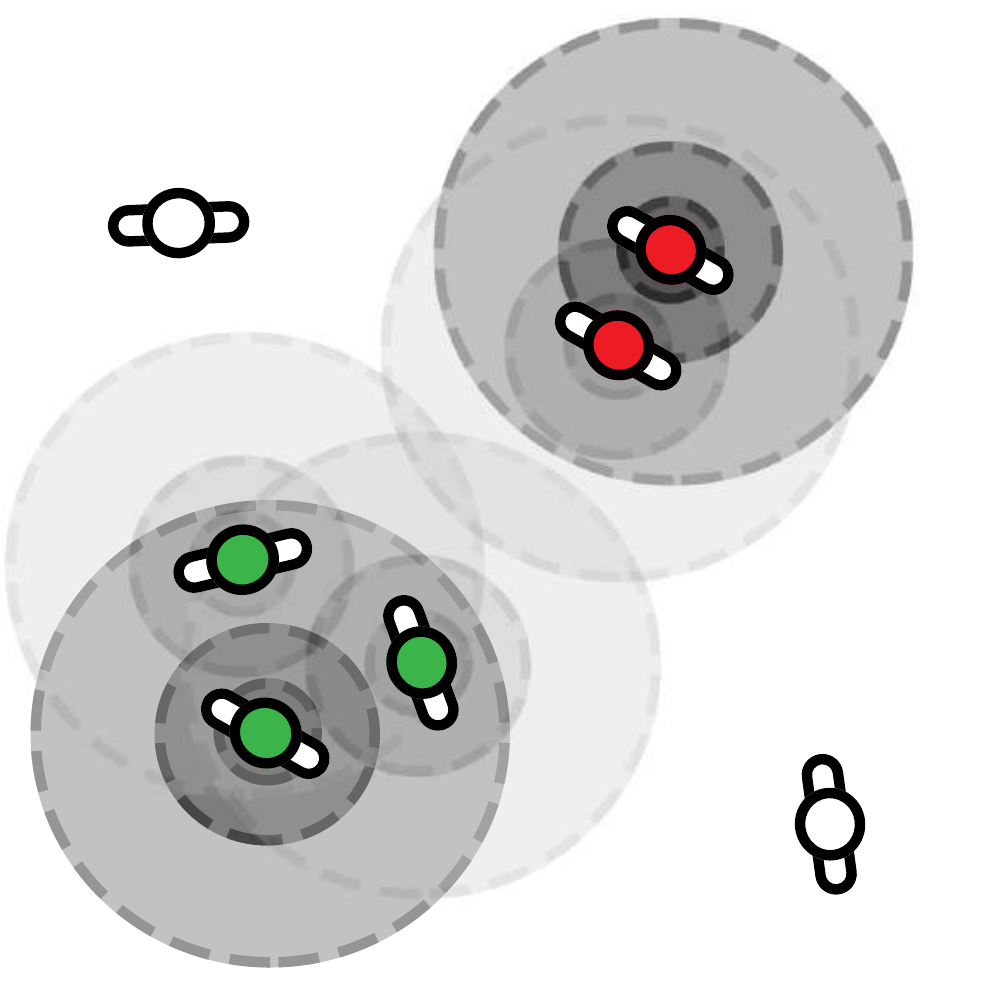}
                \label{fig:d_ph}
        }
        ~
        \subfloat[Motion causality]{
                \includegraphics[width=0.22\textwidth]{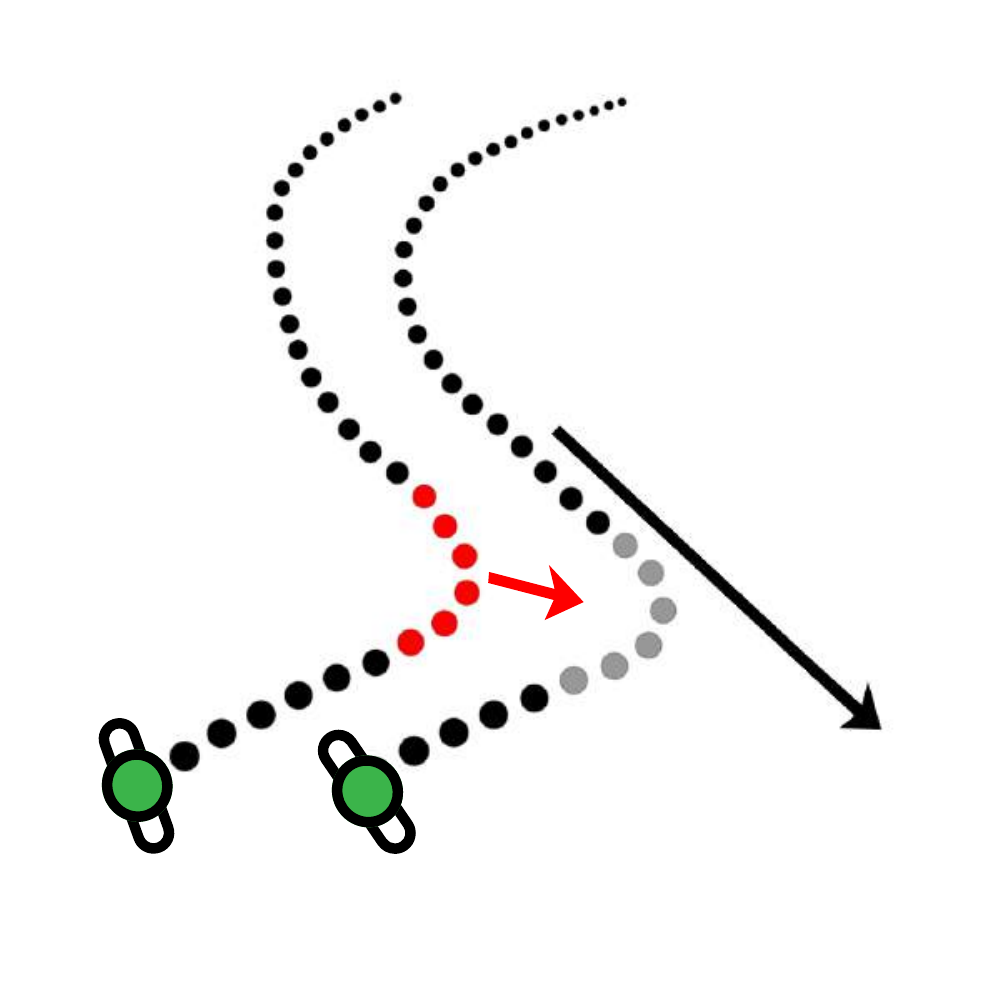}
                \label{fig:d_ca}
        }
        ~
        \subfloat[Trajectory shape]{
                \includegraphics[width=0.22\textwidth]{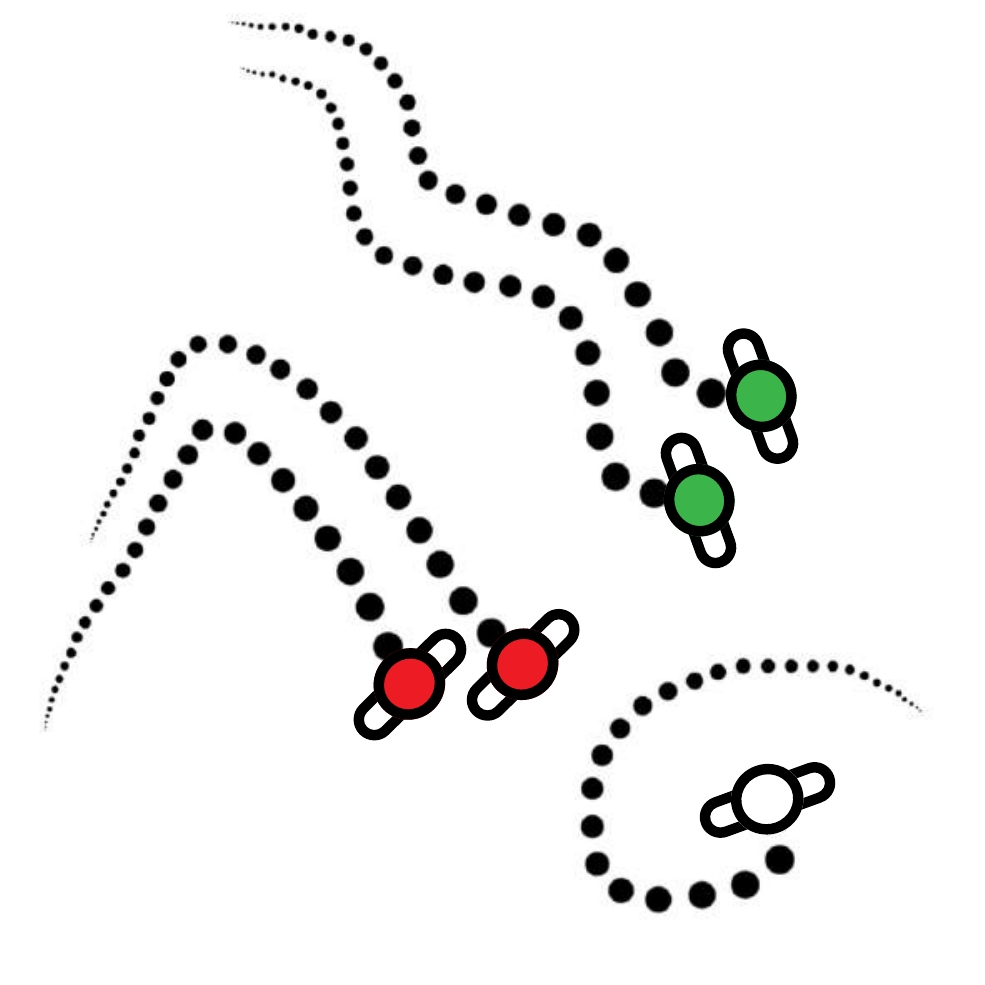}
                \label{fig:d_sh}
        }
        ~
        \subfloat[Paths convergence]{
                \includegraphics[width=0.22\textwidth]{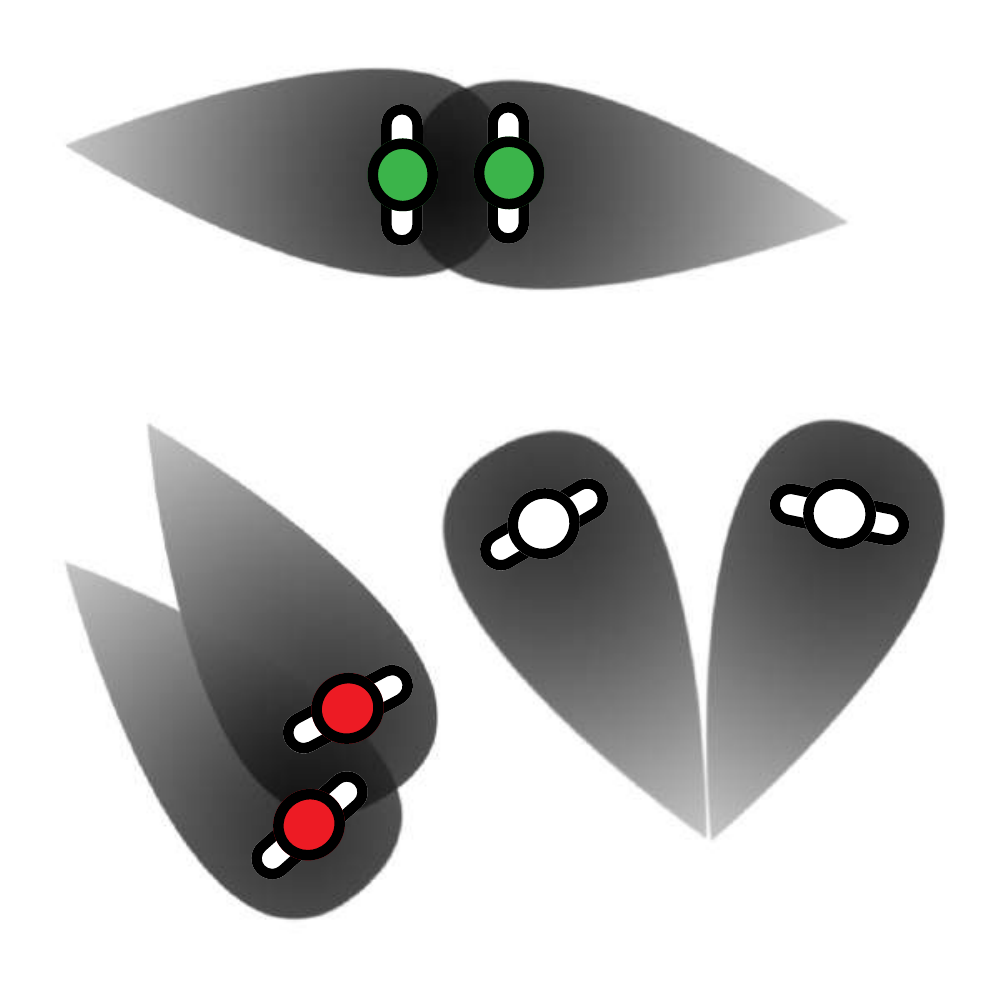}
                \label{fig:d_he}
        }
        \caption{Features: physical identity (a) and social identity (b,c) provide a computational interpretation of the concept of group membership, while (d) evaluates the likeliness of the existence of a shared goal between pedestrians.}
        \label{fig:features}
        \vspace{-0.5cm}
\end{figure*}

\section{Social Features for Social Groups}
\label{sec:features}
Given the problem formulation in Sec.~\ref{sec:problem_def} and the CC parametrization of Eq.~\eqref{eq:cc_affinity_parametrization}, here we define the distance function ${\bf d}$ which acts on trajectories pairs.
We consider the pedestrian trajectory $T_a = \{(t, {\bf p}_a^t)\}_t$, projected onto the ground plane, as multivariate time series of metric (in meters) spatial observations ${\bf p}_a^t$ for pedestrian $a$ at different times $t$.
In order to deal with the continuously changing nature of groups (splitting, merging, switching members, $\dots$) we reduce the observation period to a time window $\mathcal{T}$ of fixed length. As a consequence, groups can be differently detected even between (potentially overlapped) sequential time windows $\mathcal{T}_k$ and $\mathcal{T}_{k+1}$.

According to Def.~\ref{def:group},
%Depending on the scale at which trajectories are analyzed, different kind of information can be extracted from them. By recalling the definition of social groups from Sec.~\ref{sec:intro},
% being
%\begin{quote}
%\emph{two or more people who interact for a shared goal, perceiving a membership based on a shared social and physical identity},
%\end{quote}
we devise four features able to capture both the pedestrian physical and social identity as well as to discern the presence of a shared goal among them,
namely: \textit{physical identity} $d_\text{ph}$, \textit{trajectories shape-similarity} $d_\text{sh}$, \textit{pedestrians causality} $d_\text{ca}$ and \textit{heat-maps} $d_\text{hm}$.
%
%The first feature $d_\text{ph}$ can be regarded as the static relation that exists between pedestrians and connect physical distances to group membership.
%The remaining features are more oriented towards a dynamic understanding of groups formation, developed by considering together different observable motion clues in a specific time interval. In particular, we evaluate the social identity between group members by means of a trajectories shape similarity $d_\text{sh}$ and the extent to which the pedestrians mutually influence each other $d_\text{ca}$. Last, an heat-maps based feature $d_\text{he}$ assess whether two pedestrian are likely to share a similar goal by observing if their motion paths converge towards common scene location.
%
%The importance of dyadic relations has been recognized in sociology as the building block of social group formation~\cite{mcphail_using_1982}:
A pairwise feature vector ${\bf d}^k(a, b)$ is hence defined for every couple of trajectories $T_a$ and $T_b$ and for every time window $\mathcal{T}_k$, as
\begin{equation}
{\bf d}(a,b)\stackrel{\text{\tiny def}}{=} {\bf d}^k(a,b) = [d_\text{ph}, d_\text{sh}, d_\text{ca}, d_\text{he}]_{a,b}^k.
\end{equation}
%where the features are detailed in the rest of the section.

%\begin{figure}[tbh]
%\centering
%\includegraphics[width=\textwidth]{images/b}
%        \caption{The figure illustrates the features employed in the algorithm. Fig.~\ref{fig:d_ph} (physical identity) and Fig.~\ref{fig:d_ph}, \ref{fig:d_ca} (social identity) provide a computational interpretation of the concept of group membership, while Fig.~\ref{fig:d_he} evaluate the likeliness of the existence of a shared goal between pedestrians.} 
%        \label{fig:features}
%\end{figure}
%

\subsection{From Physical Distances to Physical Identity}
%\begin{figure}[t]
%\center
%	\subfloat[]{
%		\includegraphics[width=0.4\columnwidth]{images/GMM.pdf}
%	}
%	\subfloat[]{
%		\includegraphics[width=0.55\columnwidth]{images/prox.pdf}
%	}
%\caption{The gaussians employed in the modeling of proxemics (a) and the resulting distribution (b).}
%\label{fig:GMM}
%\end{figure}
%\begin{figure}
%	\centering
%	\includegraphics[width=0.98\columnwidth]{images/prox3.jpg}
%	\caption{Physical identity is revealed by where pedestrians mutually lay in each others proxemics bubble.}
%%	\caption{The smoothed proxemics boundaries are shown for a single pedestrians. Social relations with other pedestrians are characterized by where they lay inside the bubble. }
%	\label{fig:prox2}
%\end{figure}
\begin{figure}[t]
\center
	\subfloat[]{
		\includegraphics[width=0.60\columnwidth]{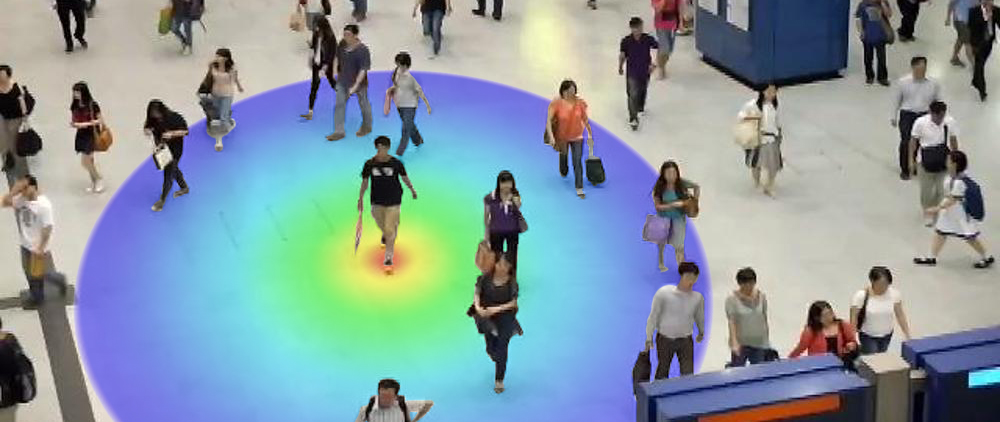}
	}
	\subfloat[]{
		\includegraphics[width=0.30\columnwidth]{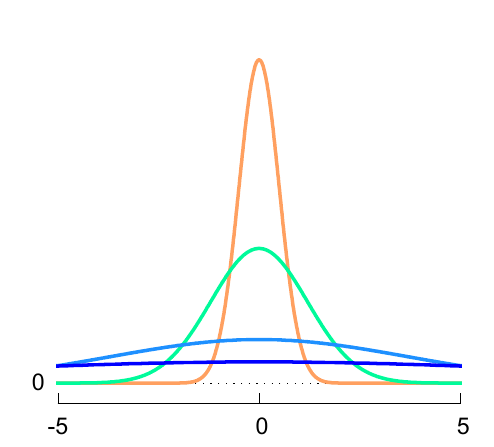}
	}
\caption{Proxemics, modeled by gaussians (b), reveal physical identity trough physical distance (a).}
\label{fig:GMM}
\end{figure}
%Physical identity is revealed by where pedestrians mutually lay in each others proxemics bubble (a) and the gaussians employed in the modeling of proxemics (b).

The \emph{physical identity} can be regarded as a static relation connecting physical distance to group membership.
In his \emph{Proxemic Theory}, Hall~\cite{hall66} focused on the physical interactions between pairs of individuals. More precisely, the theory is about ``the study of ways in which man gains knowledge of the content of other men's minds through judgments of behaviour patterns associated with varying degrees of proximity to them.''

\begin{table}[t!]
\center
\caption{Proxemics characterization as found in Hall's Theory.}
\begin{tabular}{|l|c|l|}
\hline
\multicolumn{1}{|c|}{\bfseries space} & \multicolumn{1}{c|}{\bfseries boundaries ($m$)} & \multicolumn{1}{c|}{\bfseries description}\\
\hline 
intimate & 0.0 - 0.5 & unmistakable involvement\\
\hline
personal & 0.5 - 1.2 & familiar interactions\\
\hline
social & 1.2 - 3.7 & formal relationships\\
\hline
public & 3.7 - 7.6 & non-personal interactions\\
\hline
\end{tabular} 
\label{tab:prox}
\end{table}

%This definition suggests that the relative positions among individuals are deeply rooted in the social rules of interaction.
The proxemic model fomalizes how people use physical space in interpersonal interactions and %In its practical use, the theory formalize
defines a set of concentric bubbles around every individual, as depicted in Fig.~\ref{fig:d_ph}.
%, where the interaction between pairs of individuals can be categorized based on their mutual distance.
Nevertheless, the transition between the four different proxemic zones is abrupt (Tab.~\ref{tab:prox}).
%

%
%
%\begin{itemize}
%\item intimate space (0.0-0.5$m$): unmistakable involvement;
%\item personal space (0.5-1.2$m$): familiar interactions;
%\item social space (1.2-3.7$m$): formal relationships;
%\item public space (3.0-7.6$m$): non-personal interaction.
%%\item 0.0-0.5$m$ intimate space: zone for unmistakable involvement with another body;
%%\item 0.5-1.2$m$ personal space: zone established when interacting with familiar people;
%%\item 1.2-3.7$m$ social space: zone for formal and impersonal relationships;
%%\item 3.0-7.6$m$ public space: zone for non-personal interaction with others.
%\end{itemize}
%
%The measure proposed by Hall is intrinsically quantized: every class has its boundaries defined in the metric space and the transition between them is clearly abrupt and instantaneous.
%Although effective, the spatial quantization leads to a wrong proxemic class when noise affects the measurement of people location. In real scenarios this happens frequently due to tracking errors or imprecise ground plane homographic projection. Several approaches assign a score to proxemic classes in order to obtain a continuous real-valued similarity measure, \cite{6239351,6113127,vizzari12}.
Spatial quantization can be heavily affected by noise or errors, leading to wrong classification.
% In real scenarios this happens frequently due to tracking errors or imprecise ground plane homographic projection.
Several approaches assign a score to proxemic classes in order to obtain a continuous real-valued similarity measure, \cite{6239351,6113127,vizzari12}.
To grasp the distance based characteristics of group formation, we relax the original Hall's quantization by employing a Gaussian Mixture Model (GMM) on the ground plane, centered on person location and with fixed proxemics-inspired covariance matrices.
The resulting GMM is a weighted sum of zero mean Gaussians with diagonal covariance matrices reflecting Hall's boundaries (\emph{i.e.} $\Sigma_1\leftarrow0.5$, $\Sigma_2\leftarrow1.2$, \ldots):
\begin{equation}
\text{GMM}({\bf p}_a^t-{\bf p}_b^t)=\frac{1}{4}\sum\limits_{z=1}^4 \mathcal{N}( {\bf p}_a^t-{\bf p}_b^t \vert 0,\Sigma_z)
\label{eq:GMM}
\end{equation}
Given a pair of trajectories $T_a$ and $T_b$ we evaluate the mixture model of Eq. \eqref{eq:GMM} on the vector of distances at each time instance.
This is equivalent to place the mixture on ${\bf p}_a^t$ and measure where the point ${\bf p}_b^t$ lies inside the proxemic space at each instant $t$, as shown in Fig.~\ref{fig:GMM} and in Fig.~\ref{fig:prox2}.

The static measure of social cohesion, called $d_\text{ph}$, is then defined by averaging the mixture model responses over the the set of time instances where trajectories $T_a$ and $T_b$ are simultaneously present in the current time window, $\overline{\mathcal{T}}\subseteq\mathcal{T}^k$:
\begin{equation}
d_\text{ph}^k(a,b) = \frac{1}{|\overline{\mathcal{T}}|}\sum_{t\in\overline{\mathcal{T}}}\text{GMM}({\bf p}_a^t-{\bf p}_b^t)
\end{equation}
Averaging is required since the physical identity among group members is established in time and must remain coherent in order to be a valid measure of social cohesion.
%
%\begin{figure}[tbh]
%\label{fig:GMM}
%\center
%\includegraphics[width=\textwidth]{images/GMM} 
%\caption{GMM plot}
%\end{figure}

\subsection{Motion as an Indicator of Social Identity}
\emph{Social Identity}~\cite{haslam2004psychology,turner81} is a psychological paradigm built on the intuition that group behavior is an emerging dynamic, reflecting a shift in self-conception of the members who start to define themselves in terms of their common membership.
%As a consequence, part of their self-concept is shared with other group members, enabling them to think and act as a unit.
According to~\cite{Oldmeadow05task-groupsas}, social identity
%embodies the process of implicit knowledge validation that occurs among group members and
reflects in the way people mutually influence each other and consequently move in groups, suggesting that
%While cultural and affective factors are lost during the video capturing process,
social identity could be observed through trajectories shape similarity and paths temporal causality.%, we thus propose to adopt two different motion features.
%common motion patterns manifesting social identity could be observed through trajectories shape similarity and paths temporal causality.%, we thus propose to adopt two different motion features.
%, computed on pairwise trajectories over a time window $\mathcal{T}_k$, that compare \emph{trajectories shape} and eventually their mutual time-lagged correlation, namely \emph{causality}.

\subsubsection{Temporal Causality}
%One way to unveil social identity is by observing if in any way pedestrians are mutually affecting their motion paths~\cite{couzin_effective_2005}.
%This consideration is supported in a recent work by Couzin \emph{et al.}~\cite{couzin_effective_2005} where is shown that groups (of animals) are capable of taking united decisions regarding their direction and speed of movement even when only a few members have the necessary information to make such decisions.
Under the hypothesis of sufficiently stationary trajectories, which is typically true for the observation of a time window, we can employ the econometric model of Granger causality~\cite{granger_investigating_1969} to measure to what extent pedestrians are mutually affecting their motion paths~\cite{couzin_effective_2005}. Accordingly, we formalize two requirements:
%If we dispose of a sufficiently long observation period and sufficiently stationary trajectories, 
%This concept can be mathematically modeled through a measure of causality, .
%Stating whether we may be observing actual causality between two pedestrians or rather spurious correlation, due to highly probable regularity of short trajectories, is by no means an easy task. Following the econometric model of Granger causality~\cite{granger_investigating_1969}, we can formalize two requirements:
\begin{enumerate}
\item the causal pedestrian will move before the effect pedestrian, and
\item the motion of the causal pedestrian contains information about the way the effect pedestrian moves that cannot be found in any other pedestrian motion.
\end{enumerate}
A consequence of these statements is that the causal pedestrian trajectory can help forecast the effect pedestrian trajectory even after other data has first been used. Let's define $m$ as the lag value for the causality analysis and denote the optimum least-squares predictor of a stationary trajectory $T_a$ at time $t$ using the set of values $\bar{T}_a(t-m)$ by $P_t(T_a|\bar{T}_a(t-m))$. Here $\bar{T}_a(t-m)$ is all the information about trajectory $T_a$ accumulated since time $t-m$ (inside the current time window $\mathcal{T}^k$) up to time $t-1$. The predictive error series will be denoted by $\varepsilon_t(T_a|\bar{T}_a(t-m)) = T_a(t) - P_t(T_a|\bar{T}_a(t-m))$ and define $\sigma^2(T_a|\bar{T}_a(t-m))$ as the variance of $\varepsilon_t(T_a|\bar{T}_a(t-m))$. It is said trajectory $T_b$ \emph{Granger causes} $T_a$, briefly $b\rightarrow a$, if
\begin{equation}
\sigma^2(T_a|\bar{T}_a(t-m)) > \sigma^2(T_a|\bar{T}_a(t-m), \bar{T}_b(t-m))
\end{equation}
\noindent The feature is then derived from a specific testing procedure used to evaluate Granger causality trustworthiness.
Let's introduce the sum of squared residuals for the constrained and unconstrained models as
\begin{equation}
\begin{aligned}
RSS_c &= \sum_{t=1}^{K}\varepsilon_t(T_a|\bar{T}_a(t-m))^2 \quad\text{and}\quad\\
RSS_u &= \sum_{t=1}^{K}\varepsilon_t(T_a|\bar{T}_a(t-m), \bar{T}_b(t-m))^2,
\end{aligned}
\end{equation}
where $K$ is the number of samples considered for the analysis.
\begin{figure}
	\centering
	\includegraphics[width=\columnwidth]{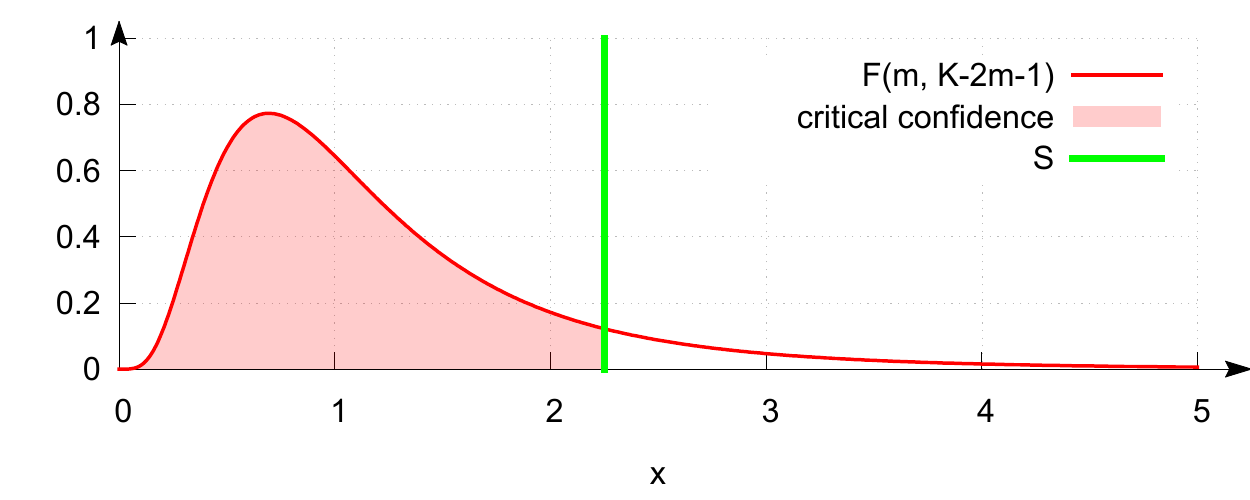}
	\caption{Visual example of causality probability. The vertical line is the $S$ of Eq.~\eqref{eq:ca_S} while the shaded area is $d_\text{ca}$.}
	\label{fig:caus_S}
\end{figure}
We design our feature $d_\text{ca}$ so as to be the critical confidence measure of the hypothesis that Granger causality exists between $T_a$ and $T_b$. To this end, we consider the test statistic
\begin{equation}
\label{eq:ca_S}
S_{b\rightarrow a} = \frac{(RSS_c-RSS_u)/m}{RSS_u/(K-2m-1)}.
\end{equation}
% \sim F_{m, K-2m-1}.
and compute the area under the Fisher-Snedecor probability function $\mathcal{F}$ to the left of $S$, as shown in Fig.~\ref{fig:caus_S}. This results in the following closed form solution~\cite{hazewinkel_encyclopaedia_1989} integral:
%and compute the area under the probability distribution to the left of the value where Fisher-Snedecor $\mathcal{F}$ function equals $S$~\cite{hazewinkel_encyclopaedia_1989}, as shown in Fig.~\ref{fig:caus_S}. More formally:
\begin{equation}
d^k_\text{ca}(a,b) = \max_{S\in\{S_{b\rightarrow a}, S_{a\rightarrow b}\}}\int_{0}^{S} \mathcal{F}(x\vert m, K-2m-1)\mathrm{d}x,
\end{equation}
where $S_{b\rightarrow a}$ and $S_{a\rightarrow b}$ are both considered in order to obtain symmetry, but as we value the existence of causality over its direction, we only keep the one which maximize the probability.

\subsubsection{Shape Similarity}
Shape similarity may also be useful in describing social identity as it overcomes the limit of the proxemics punctual and static evaluation.
%Commonly, time series are unaligned (due two different timings in performing the path) and mostly of different lengths.
We use the Dynamic Time Warping (DTW)~\cite{berndt_using_1994} on euclidean coordinates to map one time series to another by minimizing the distance between the two. In particular, DTW flexibility allows two time series that are similar but locally out of phase to align in a non-linear manner.
%The non-linear mapping implicitly addresses the problem of different series lengths as well, by mapping multiple elements of a time series to the same element of the other one.
Suppose we have two trajectories $T_a$ and $T_b$ of lengths $§A$ and $B$ respectively. 
To align these two sequences using DTW, we first construct a distance matrix $\{D^{ij}_{ab}\}_{ij}\in\mathbb{R}^{A\times B}$ that encodes the squared euclidean distance between any $i$-th element of $T_a$ and $j$-th element of $T_b$ inside the current time window.
%$D^{ij}_{ab}$ can thus be interpreted as the alignment cost between the $i$-th element of $T_a$ and the $j$-th element of $T_b$.

The best alignment can be found by a recursive minimization of the cumulative cost $\gamma_{ab}$ of any path through the distance matrix originating in $D_{ab}^{11}$:
\begin{equation}
\gamma_{ab}(i, j) = D^{ij}_{ab} + \min\{\gamma_{ab}(i\text{-}1, j), \gamma_{ab}(i\text{-}1, j\text{-}1), \gamma_{ab}(i, j\text{-}1)\}.
\end{equation}
In particular, we construct our feature to be the distance of the two sequences once they are optimally aligned, that is the sum of the Euclidean distances of associated points of $T_a$ and $T_b$:
\begin{equation}
d_\text{sh}(a,b) = \gamma_{ab}(A, B)/\max(A,B)
\end{equation}
where the denominator is the optimal warping path length used as a normalization factor.

\subsection{Common Goals from People Motion}
\label{sec:heatmaps}
Previously described features focus on both static and dynamic aspect of trajectories when groups are already established, but neglect the smooth process of group formation. People may merge in groups starting from different location (\emph{e.g.} meeting action) or groups may split into subgroups and singletons (according to the \textit{hierarchical coherence} property of group formation).
Meeting or being close for a sufficient amount of time may indicate the presence of a shared goal. Following the results in \cite{lin_heat-map-based_2013}, where heat maps were used to recognize group activities, we also employ a heat map inspired feature to holistically model groups.

A heat map $H_a:\mathbb{N}_R\times\mathbb{N}_C\rightarrow[0,1]$ associated to the trajectory $T_a$ is a $R$-by-$C$ grid of heat sources $h_a$ that partitions the ground plane. The heat source $h_a(i,j)$ activates if the trajectory $T_a$ happens to walk in the relative grid cell $(i,j)$ and once activated it is subject to thermal decay and thermal diffusion processes:
\begin{equation}
H_a(i,j) = \sum_{p=1}^R\sum_{q=1}^C E_a(p, q) \cdot e^{-k_s\|(p-i,q-j)\|},
\end{equation}
where $k_s$ is a parameter suggesting the relative importance of different patches at different distances and $E_a(p, q)$ is the thermal energy produced by $T_a$ on the patch $(p, q)$. If we let $\bar{E}_a(p, q)$ be the accumulated thermal energy, we have
\begin{equation}
E_a(p, q) = \bar{E}_a(p, q)\cdot e^{-k_rt_\text{int}},
\end{equation}
being $k_r$ a parameter regulating the slow down of the heat accumulation and dispersion and $t_\text{int}$ the duration of the interaction between pedestrian $a$ and cell $(p,q)$ inside the current time window $\mathcal{T}^k$.

\begin{figure}[t!]
\centering
	\includegraphics[width=0.93\columnwidth]{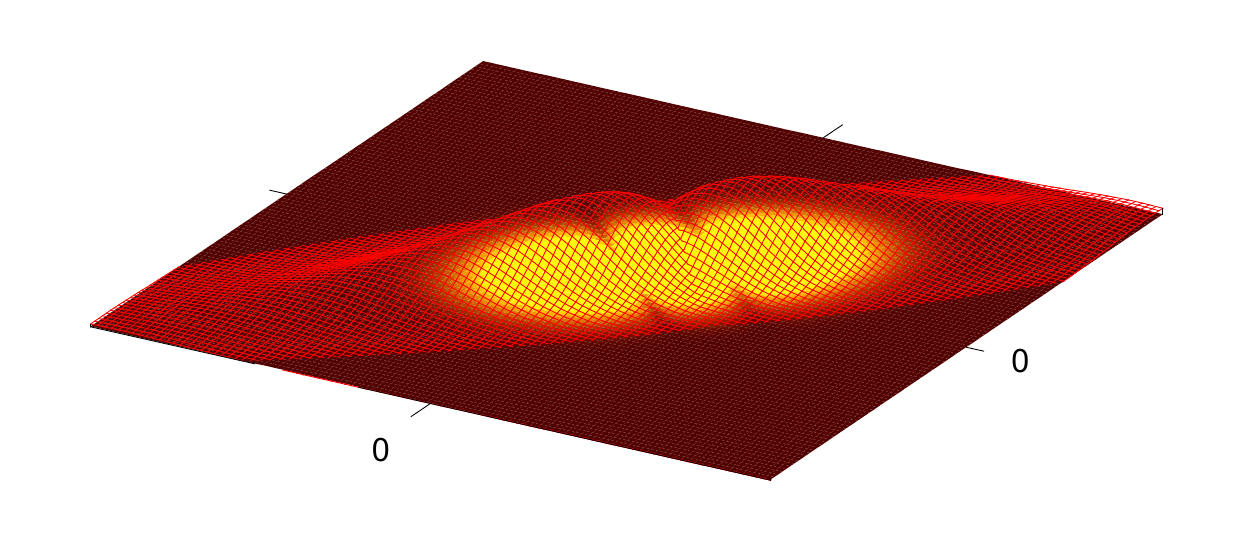}
	\caption{Intersecting heat maps are generated by converging trajectories, which project on the $xy$ plane their shared goal.}
	\label{fig:hm_feature}
\end{figure}

Once we have constructed heat maps for every trajectory, we define a similarity metric between two trajectories $T_a$ and $T_b$ as the volume under the combined heat surface $\Upsilon_{ab}$ obtained as the pointwise product of the two heat maps $H_a$ and $H_b$:
\begin{equation}
d^k_\text{he}(a,b) = \sum_{i=1}^R\sum_{j=1}^C \Upsilon_{ab}(i,j) = \sum_{i=1}^R\sum_{j=1}^C H_a(i,j)H_b(i,j)
\end{equation}
The volume under $\Upsilon_{ab}$ reveals to what extent $T_a$ and $T_b$ have been close in space during the observation period, something that proxemics could already measure indeed. Nevertheless, heat maps relax the constraint by which only elements from the same frame can be compared, in practice this is accomplished through the thermal diffusion process.
%Note that heat maps do not employ any kind of frame-by-frame comparison and the just defined similarity measure is an holistic metric.
At the same time, heat maps also expose the history of their respective trajectories, allowing the metric to capture the temporal aspect of motion similarity.
Proxemics, DTW and Granger causality would rate two pedestrians meeting and parting ways analogously, even if the former case is more likely to represent a group formation process.
%As an example, consider two distinct scenario were a pair of pedestrian are meeting and parting respectively. Proxemics, DTW and Granger causality would rate both these scenarios analogously even if two pedestrians getting closer are more probably to be considered part of the same group than two going in opposite directions.
Recognizing motion trajectories also encode temporal information is a great advantage of heap maps based analysis.

%%%%%%%%%%%%%%%%%%%%%%%%%%%%%%%%%%%%%%%%%%%%%%%%%%%%%%%%%%%%%%%%%%%%%%%%%%%%%%%%%%%%%%%%%%%%%%%%%%%%
\section{Learning Framework}
\label{sec:learning}
The linear parametrization of the affinity matrix $W_{\bf d}$ of Eq.~\eqref{eq:cc_affinity_parametrization} guarantees to reach a partition of the crowd which is consistent with the social groups properties. The parameters ${\bf w} = [{\boldsymbol\alpha}, {\boldsymbol\beta}]$ govern both the importance of each feature alone and their similarity/dissimilarity optimal combinations, resulting in different clustering rules.

The choice of the best rule should account for all factors affecting the group formation process, such as environmental constraints or cultural influences. 
The complexity of explicitly evaluating these factors resides in the impossibility to directly observe them. Still, we can gain important insights by observing the grouping process. On these premises, we adopt a learning framework capable of choosing the most suitable clustering rule by finding a set of feature weights that implicitly embodies these non-observable aspects.

%Recall that the affinity matrix $W_d$ employed by the correlation clustering is parameterized through ${\bf w} = [\alpha, \beta]$, as shown in Eq.~\eqref{eq:cc_affinity_parametrization}. These parameters govern both the importance of each feature alone and their similarity/dissimilarity optimal combinations. While it's true that every non-degenerate choice of those parameters would produce a clustering result consistent to the axioms of social groups, different values of ${\bf w}$ can result in different clustering rules. The choice of the best rule has to be made according to a sequence of training videos, enabling us to learn the concept of groups in that specific context.

\subsection{Supervised CC Through Structured Learning}
%By learning the weight vector ${\bf w}$ from the data, we are able to leave the features unchanged but combine them in the most effective way to discriminate groups in the current scenario.

Let us consider the input ${\bf x}_i = \{[{\bf 1} - {\bf d}^i(a,b); {\bf d}^i(a,b)]\}_{a,b}$ to be the set of pairwise features computed on all the possible pairs of trajectories $T_a$ and $T_b$ in the $i$-th temporal window and ${\bf y}_i$ the clustering solution, \emph{i.e.} the set of all social groups appearing in the crowd $M_i$. Since ${\bf y}_i$ cannot be described by a single valued function, we adopt the
%It can be observed that the ${\bf y}_i$ is not described by a single valued function but it is inherently structured and its elements interdependent. %, and as such we cannot use neither a traditional classifier nor a regressor.
Structural SVM~\cite{tsochantaridis_large_2005} framework to model and learn predicting the solution.
%structured outputs by solving a loss augmented problem.
The goal is to learn a classification mapping $f:\mathcal{X}\rightarrow\mathcal{Y}$ between input space $\mathcal{X}$ and structured output space $\mathcal{Y}$ given a set of input-output pairs $\{({\bf x}_1, {\bf y}_1),\dots,({\bf x}_n, {\bf y}_n)\}$.
A discriminant score function $F:\mathcal{X}\times\mathcal{Y}\rightarrow\mathbb{R}$ is defined over the joint input-output space and $F({\bf x}, {\bf y})$ can be interpreted as measuring the compatibility of ${\bf x}$ and ${\bf y}$. Now, the prediction function $f$ can be defined as
\begin{equation}
\label{eq:pred_fun}
f({\bf x})~=~\arg\max_{{\bf y}\in\mathcal{Y}({\bf x})}F({\bf x}, {\bf y})
\end{equation}
where the maximizer over the label space $\mathcal{Y}({\bf x})$ is the predicted label, \emph{i.e.} the solution of the inference problem. 
For simplicity we choose to restrict the space of $F$ to linear functions over some combined feature representation $\Psi({\bf x}, {\bf y})$ subject to a ${\bf w}$ parametrization. This feature mapping cannot be defined out of the context of the problem, as it is the problem itself that specifies, given a particular input, the nature of the desired solution. Following the definition of correlation clustering in Eq.~\ref{eq:correlation_clustering_objective} and its parametrization introduced in Eq.~\ref{eq:cc_affinity_parametrization}, the compatibility of an input-output pair is directly described as
\begin{equation}
F({\bf x}, {\bf y}; {\bf w}) = {\bf w}^T\Psi({\bf x}, {\bf y}) = {\bf w}^T\sum_{y\in{\bf y}}\sum_{a\neq b\in y}{\bf x}^{ab}.
\end{equation}
The problem of learning in structured and interdependent output spaces can been formulated as a maximum-margin problem. We adopt the $n$-slack, margin-rescaling formulation:
\begin{equation}
\begin{aligned}
& \min_{{\bf w}, {\bf \xi}}
& & \frac{1}{2}\|{\bf w}\|^2+\frac{C}{n}\sum_{i=1}^n\xi_i \\
& \,\,\,\,\,\text{s.t.}
& & \forall i:\xi_i\ge0, \\
&&& \forall i,\forall{\bf y}\in\mathcal{Y}({\bf x}_i)\backslash{\bf y}_i:{\bf w}^T\delta\Psi_i({\bf y})\ge\Delta({\bf y}, {\bf y}_i)-\xi_i,
\end{aligned}
\label{optpro}
\end{equation}
where $\delta\Psi_i({\bf y}) \overset{\text{def}}{=} \Psi({\bf x}_i, {\bf y}_i) - \Psi({\bf x}_i, {\bf y})$, $\xi_i$ are the slack variables introduced in order to accommodate for margin violations, $\Delta({\bf y}_i, {\bf y})$ is the loss function further defined in Sec.~\ref{sec:loss_score} and $C$ is the regularization trade-off. Intuitively, we want to maximize the margin and jointly guarantee that for a given input, every possible output result is considered  worst than the correct one by at least a margin of $\Delta({\bf y}_i, {\bf y})-\xi_i$, where $\Delta({\bf y}_i, {\bf y})$ is bigger when the two predictions are known to be more different.

Remarkably, correlation clustering doesn't need to know in advance how many groups are present in the scene. Moreover, a positive overall cluster score can group two elements even if their affinity measure is negative, implicitly modeling the transitive property of relationships in groups, as stated in Sec.~\ref{sec:problem_def}.

\subsection{Batch Sequential Optimization}
The quadratic program (QP)~\eqref{optpro} introduces a constraint for every possible wrong clustering of the $n$ examples, more precisely $\sum_{i=1}^n(|\mathcal{Y}({\bf x}_i)|-1)$. Unfortunately, the number of ways to partition a set $M$ scales more than exponentially with the number of items according to the Bell sequence~\cite{rota_number_1964}
\begin{equation}
|\mathcal{Y}(M)| = \sum_{i=0}^{|M|}\frac{1}{i!}\sum_{j=0}^i(-1)^{i-j}{k\choose j} j^{|M|},
\end{equation}
making the optimization intractable. As an example, for a crowd composed of 20 pedestrians the number of potential solutions would be about $5.8\cdot 10^{12}$. In order to deal with this high number of constraints many approximation schemes have been proposed, where cutting plane algorithms or subgradient methods
%(\emph{e.g.} \cite{joachims_cutting-plane_2009, shalev-shwartz_proximal_2012})
are among the most commonly used. In particular, all the constraints of QP~\eqref{optpro} can be replaced by $n$ piecewise-linear ones by defining the structured hinge-loss:
\begin{equation}
\widetilde{H}({\bf x}_i) \overset{\text{def}}{=} \max_{{\bf y}\in\mathcal{Y}}\Delta({\bf y}_i, {\bf y}) - {\bf w}^T\delta\Psi_i({\bf y}).
\label{eq:maxoracle}
\end{equation}
The computation of the structured hinge-loss for each element $i$ of the training set, described in Sec.~\ref{sec:oracle}, amounts to finding the most ``violating'' output ${\bf y}$ for a given input ${\bf x}_i$ and its correct associated output ${\bf y}_i$.
We only have $n$ constraints of the form $\xi_i \geq \tilde{H}({\bf x}_i)$ and the non-smooth version of QP~\eqref{optpro} reduces to
\begin{equation}
\begin{aligned}
& \min_{{\bf w}}
& & \frac{1}{2}\|{\bf w}\|^2+\frac{C}{n}\sum_{i=1}^n\widetilde{H}({\bf x}_i).
\end{aligned}
\label{optpro_unconstrained}
\end{equation}
By disposing of a maximization oracle, \emph{i.e.} a solver for Eq.~\eqref{eq:maxoracle}, and a computed solution ${\bf y}^*$, 
subgradient methods can easily be applied to QP~\eqref{optpro_unconstrained}, being $\partial_{\bf w}\widetilde{H}({\bf x}_i) = -\delta\Psi_i({\bf y}^*)$.
% and ${\bf y}^*$ is any maximizer of Eq.~\eqref{eq:maxoracle}.
%It is easy to see that the subgradient of $\widetilde{H}({\bf x}_i)$ with respect to ${\bf w}$ is $-\delta\Psi_i({\bf y}^*)$, where ${\bf y}^*$ is any  maximizer for Eq.~\eqref{eq:maxoracle}.

\begin{algorithm}
\setstretch{1.35}
\caption{Block-Coordinate Frank-Wolfe Algorithm}
\label{BCFW}
\begin{algorithmic}[1]
    \STATE Let ${\bf w}^{(0)}, {\bf w}_i^{(0)} := {\bf 0} $ and $ l^{(0)}, l_i^{(0)} := 0$
	\FOR{$\text{it} := 0$ \TO $\text{maxIterations}$ }
		\STATE Pick $i$ at random in $\{1, \dots, n\}$
		\STATE Solve ${\bf y}* := \arg\max_{{\bf y}\in\mathcal{Y}}\Delta({\bf y}_i, {\bf y}) - {\bf w}^T\delta\Psi_i({\bf y})$
		\STATE Let ${\bf w}_s := \frac{C}{n}\delta\Psi_i({\bf y}^*)$ and $l_s := \frac{C}{n}\Delta({\bf y}_i, {\bf y}^*)$
		\STATE Let $\gamma := \frac{({\bf w}_i^{(\text{it})}-{\bf w}_s)^T{\bf w}^{(\text{it})}+\frac{C}{n}(l_s-l_i^{(\text{it})})}{|{\bf w}_i^{(\text{it})}-{\bf w}_s\|^2}$ and clip to $[0,1]$
		\STATE Update ${\bf w}_i^{(\text{it}+1)} := (1-\gamma){\bf w}_i^{(\text{it})} + \gamma {\bf w}_s$ \\$\quad$ and $l_i^{(\text{it}+1)}:= (1-\gamma)l_i^{(\text{it})}+\gamma l_s$
		\STATE Update ${\bf w}^{(\text{it}+1)}:= {\bf w}^{(\text{it})} + {\bf w}_i^{(\text{it}+1)} - {\bf w}_i^{(\text{it})}$\\ $\quad$ and $l^{(\text{it}+1)} := l^{(\text{it})} + l_i^{(\text{it}+1)}-l_i^{(\text{it})}$
	%	\frac{\lambda({\bf w}_i^{(k)}-{\bf w}_s)^T{\bf w}^{(k)}-l_i^{(k)}+l_s}{\lambda\|{\bf w}_i^{(k)}-{\bf w}_s\|^2}
	\ENDFOR
\end{algorithmic}
\end{algorithm}

To exploit the domain separability of the constraints and limit the number of oracle calls needed to converge to the optimal solution, we choose to adopt a Block-Coordinate version of the Frank-Wolfe algorithm (BCFW)~\cite{julien_block_coordinate_2012}, delineated in Alg.~\eqref{BCFW}.
%In this paper we consider a recent Block-Coordinate version of the Frank-Wolfe algorithm by Lacoste-Julien~\emph{et al.}~\cite{julien_block_coordinate_2012} (BCFW), delineated in Alg.~\eqref{BCFW}, which exploits the domain separability of the constraints to limit the number of oracle calls needed to converge to the optimal solution ${\bf w}$.
%The main idea of the method is to optimize the objective function using only an example at a time, rather than the whole dataset. This approach is motivated by gradient descent methods often successfully employed in large scale optimization. The key advantage of BCFW over stochastic gradient descent algorithms is that there is a closed form solution to optimize the step size through an analytic line-search (line 6 of Alg.~\ref{BCFW}).
The algorithm works by minimizing the objective function of Eq.~\eqref{optpro_unconstrained} but restricted to a single random example at each iteration. By calling the max oracle upon the selected training sample (line 4) we obtain a new sub-optimal parameter set ${\bf w}_s$ by simple derivation (line 5).
% which is guaranteed to be inside the feasible region and as such does not require any projection step.
The best update is then found through a closed-form line search (line~6), greatly reducing convergence time compared to other subgradient methods.\\

\noindent In order to solve QP~\eqref{optpro_unconstrained} effectively, it is important to choose an appropriate loss function as the learning ability of Structural SVM highly depends on it. In Sec.~\ref{sec:loss_score} we introduce and discuss different potential loss functions and their respective descriptive ability. Given the loss function, in Sec.~\ref{sec:oracle} an efficient method to compute the maximization oracle (line 4 of Alg.~\ref{BCFW}) is described.

\subsection{Loss Function and Scoring Procedure}
\label{sec:loss_score}

% The learning ability of the Structural SVM highly depends on the choice of the loss function $\Delta({\bf y}_i, {\bf y})$.
% since it has the power to force or relax margins in the joint feature space.

One common choice of loss function for clustering is the \emph{pairwise loss} $\Delta_{PW}({\bf y}_i, {\bf y})$, which is a generalization of the Rand coefficient~\cite{rand_objective_1971}, and is defined as the ratio between the number of pairs on which ${\bf y}_i$ and ${\bf y}$ disagree on their cluster membership and the number of all possible pairs of elements in the set.
Due to the quadratic number of connections that exist among crowd members, this measure tends to be imprecise when dealing with large crowds: as the crowdness increases, the number of positive links connecting group members becomes negligible with respect to the total number of links. As a consequence, erroneous solutions won't be strongly penalized.\\ %whenever a group split occurs, the pairwise loss tends to preserve the largest amount of inner links among members by detaching singletons; this behavior is in contrast with the hierarchical coherence property.\\
%The main strengths of this measure are its low computational cost and its simple analytic description but it tends to be imprecise when dealing with groups in large crowds.
%This is due to the quadratic number of connections that exist among crowd members: for low crowded scenes this measure seems to behave properly but, as the crowdness increases, the pairwise loss become inconclusive because the number of positive links connecting group members (which is quadratic in the number of members of each group) becomes negligible with respect to the total number of links. As a consequence, whenever a group split occurs, the pairwise loss tends to preserve the largest amount of inner links among members by detaching singletons; this behavior is in contrast with the hierarchical coherence property.\\ 
The \emph{MITRE~loss}~\cite{vilain_model-theoretic_1995}, $\Delta_{M}({\bf y}_i, {\bf y})$, founded on the understanding that connected components are sufficient to describe groups, partially mitigates this problem by representing groups as spanning trees, instead of complete graphs, inducing a linear amount of both positive and negative links among members (and not quadratic as in the pairwise case).
%The measure was originally defined for the noun-coreference problem in NLP and often adopted for different clustering tasks as well. As a matter of fact, the noun-coreference problem is in many ways similar to crowd partitioning, in particular the two problems share homomorphic combinatorial solution spaces.
%The \emph{MITRE~loss}, $\Delta_{M}({\bf y}_i, {\bf y})$, is founded on the understanding that connected components are sufficient to describe groups and thus spanning trees instead of complete graphs can be used to represent clusters. This is related to the transitivity axiom of social groups described in Sec.~\ref{sec:problem_def}, which states that a group connectiveness should be computed by considering the net value of all the paths between each member and all the other ones. \\
%Starting with a group, we can obtain a spanning tree by removing all the arcs that form a cycle in the original complete graph of elements.
For any crowd partitioning, a spanning forest is an equivalence class as many trees that describe the same group configuration may exist.
The final score is obtained by accounting for the number of links that needs to be removed or added to recover a spanning forest of the correct solution.
Nonetheless, problems arise when working on relations and not directly on members, as singletons have no connections at all but should still be considered positively when correctly classified.

\begin{figure}[t!]
	\centering
        \subfloat[${\bf y}_i$ PAIRWISE links]{
                \includegraphics[width=0.45\columnwidth]{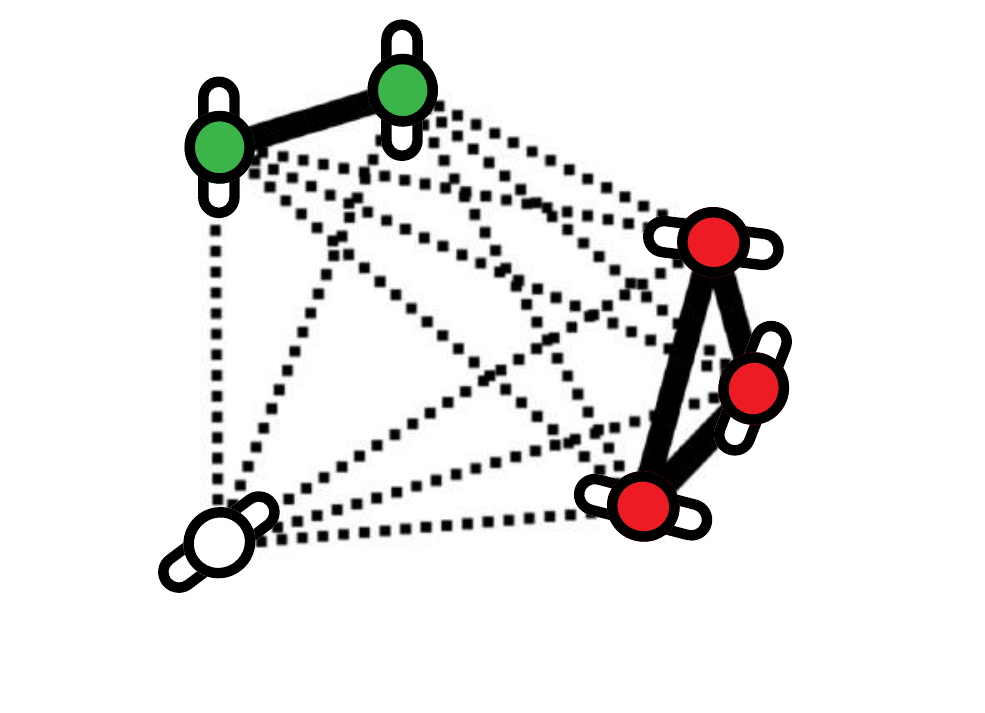}
        }
        \subfloat[${\bf y}, \Delta_{PW}({\bf y}_i, {\bf y})=0.27$]{
                \includegraphics[width=0.45\columnwidth]{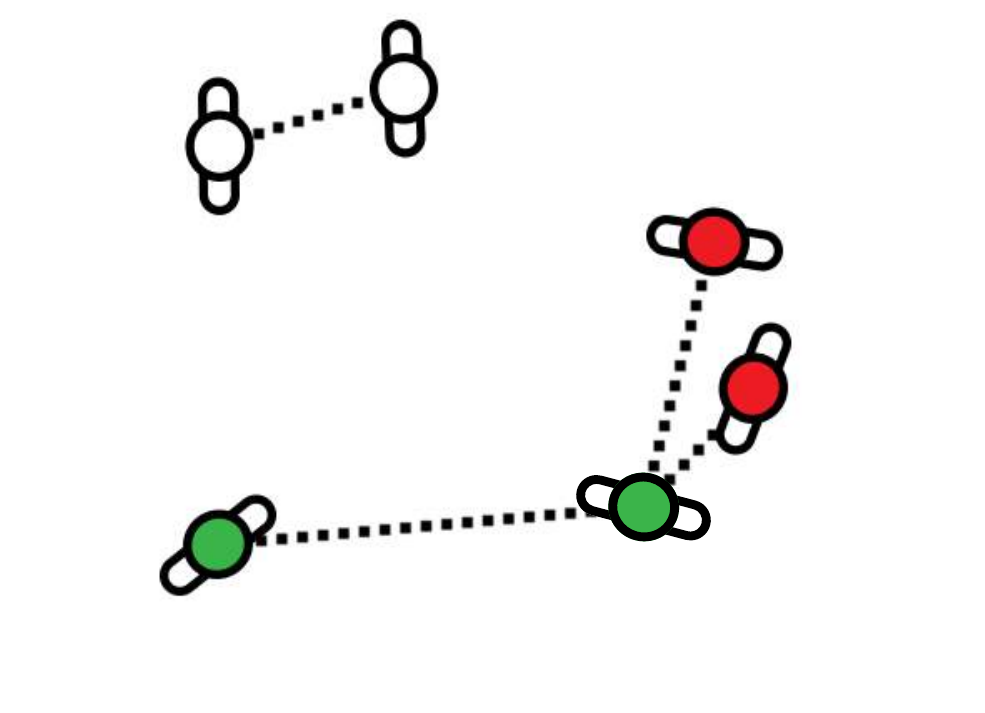}
        }\\
        \subfloat[${\bf y}_i$ MITRE links]{
                \includegraphics[width=0.45\columnwidth]{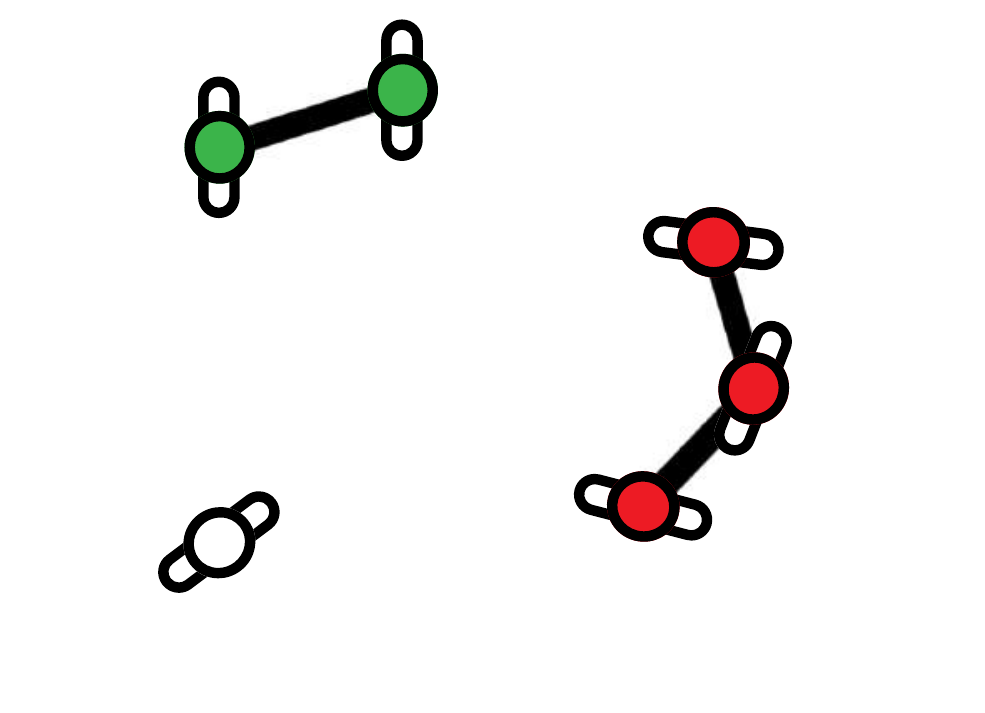}
        }
        \subfloat[${\bf y}, \Delta_{M}({\bf y}_i, {\bf y})=0.6$]{
                \includegraphics[width=0.45\columnwidth]{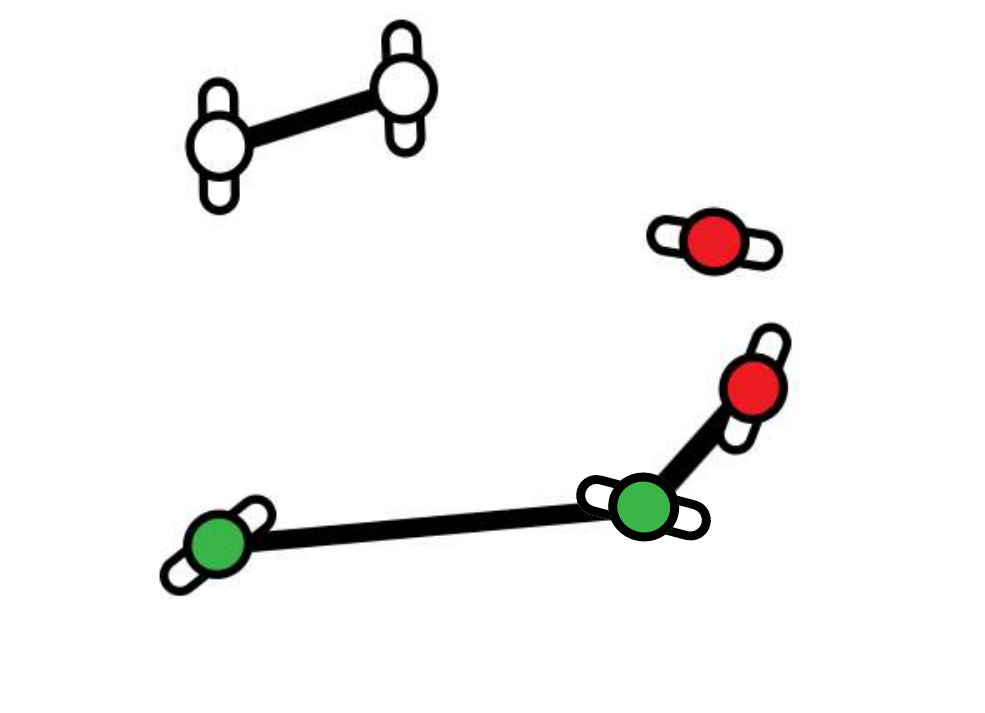}
        }\\
        \subfloat[${\bf y}_i$ $G$-MITRE links]{
                \includegraphics[width=0.45\columnwidth]{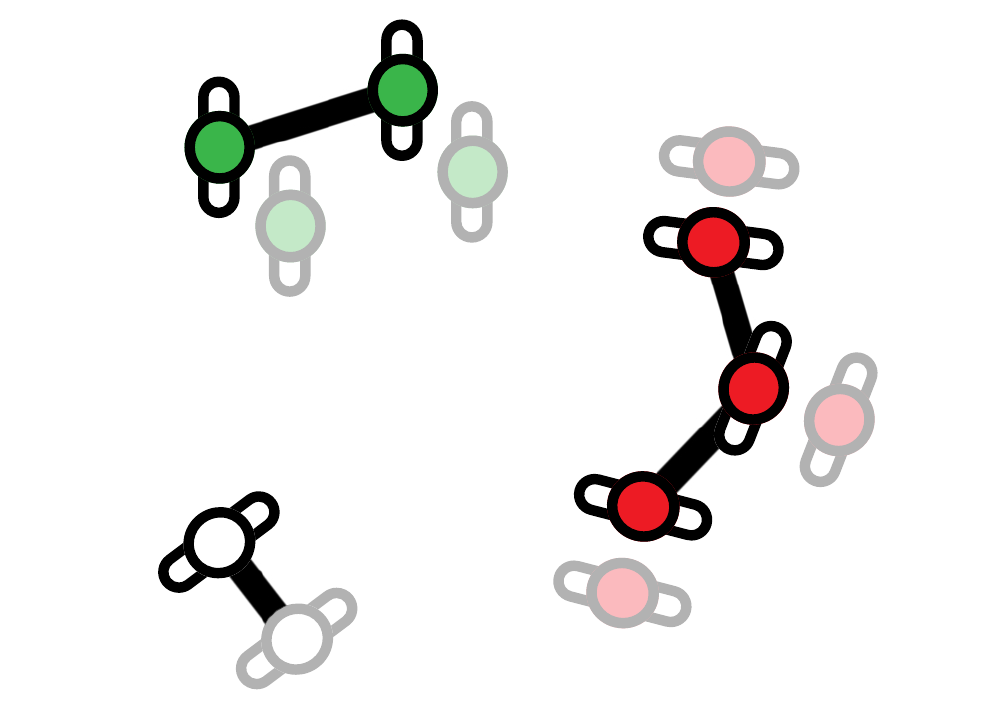}
        }
        \subfloat[${\bf y}, \Delta_{GM}({\bf y}_i, {\bf y})=0.75$]{
                \includegraphics[width=0.45\columnwidth]{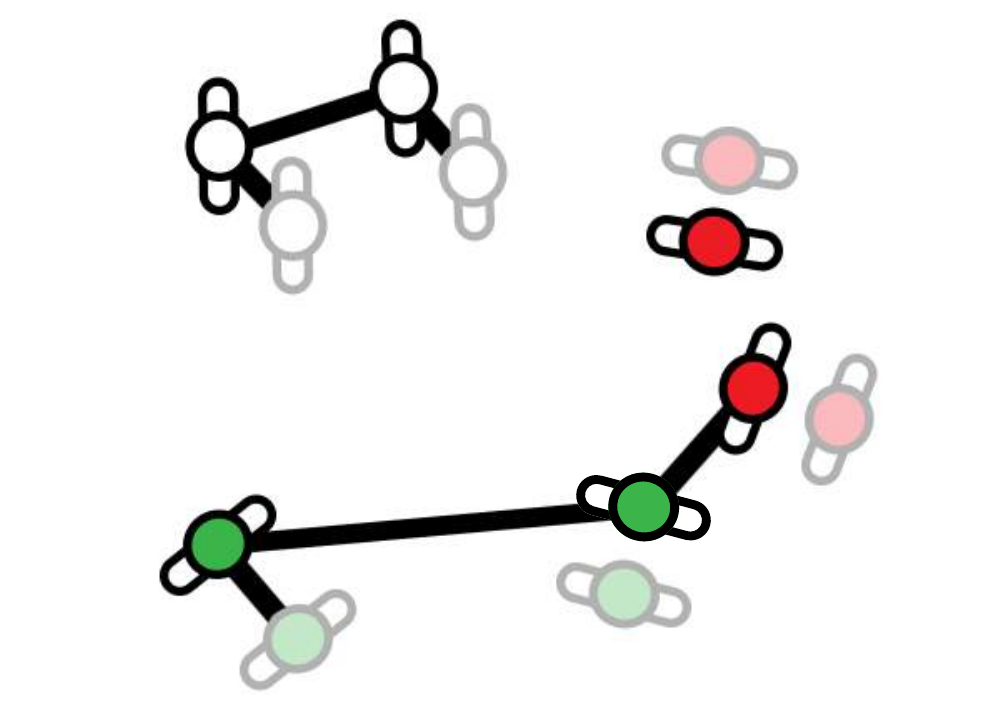}
        }
        \caption{Differences in the way losses account for errors. Singletons are white. Figures (a, c, e) depict solution ${\bf y_i}$ and the links considered by the respective losses, while (b, d, f) color pedestrians according to solution ${\bf y}$ and show the links on which the two solutions ${\bf y}_i$ and ${\bf y}$ disagree.}
        \label{fig:losses}
\end{figure}

For this motivation, we propose a loss function, \emph{GROUP-MITRE loss} ($G$-MITRE) $\Delta_{GM}({\bf y}_i, {\bf y})$, that overcomes this limitation by adding, for each pedestrian described by the trajectory $T_i$, a fake counterpart $\alpha_{T_i}$ to which only singletons are connected.
Through this shrewdness we can now take into consideration singletons as well when computing the discrepancy between two solutions. The particular design choice to link to the fake counterparts only singleton members generates two discrepancies when committing errors involving singletons and is thus a further effort in generating more plausible hierarchical groups in the solution, as depicted in Fig.~\ref{fig:losses}.
More formally, consider two clustering solutions ${\bf y}_i$, ${\bf y}$ and a representative of their respective spanning forests $Q$ and $R$. The connected components of $Q$ and $R$ are identified respectively by the set of trees $Q_{1}, Q_{2}, \dots$ and $R_1,R_2,\dots$. Note that if the number of elements in $Q_j$ is $|Q_j|$, then only $c(Q_j)\stackrel{\text{\tiny def}}{=}|Q_j|-1$ links are needed in order to create a spanning tree. Let us define $\pi_{\scriptscriptstyle R}(Q_j)$ as the partition of a tree $Q_j$ with respect to the forest $R$, that is the set of subtrees obtained by considering only the membership relations in $Q_j$ also found in $R$. Besides, if $R$ partitions $Q_j$ in $|\pi_{\scriptscriptstyle R}(Q_j)|$ subtrees then $v(Q_j)\stackrel{\text{\tiny def}}{=}|\pi_{\scriptscriptstyle R}(Q_j)| - 1$ links are sufficient to restore the original tree. It follows that the recall error for $Q_j$ can be computed as the number of missing links divided by the minimum number of links needed to create that spanning tree. Accounting for all trees $Q_j$ the global recall measure of $Q$ is:
\begin{equation}
\label{eq:recall_mitre}
\begin{aligned}
\mathcal{R}_{Q} = 1 - \frac{\sum_{j} v(Q_j)}{\sum_{j} c(Q_j)} = \frac{\sum_{j} |Q_j|- |\pi_{\scriptscriptstyle R}(Q_j)|}{\sum_{j}|Q_j|-1}
\end{aligned}
\end{equation}
The precision of $Q$ (recall of $R$) can be computed by exchanging $Q$ and $R$. Given the definition of precision, recall and employing the standard $F$-score $F_1$, the loss is defined as
\begin{equation}
\Delta_{GM}=1-F_1.
\end{equation}
%&= 1- 2\frac{\mathcal{R}_S\mathcal{R}_R}{\mathcal{R}_S+\mathcal{R}_R}
%where $F_1$ is the standard $F$-score.\\
%
%Nonetheless this measure turns out to be imprecise when applied to the group detection task since, by working on relations and not directly on members, it can't deal with singletons which have no connections at all but should still be considered positively when correctly classified. Our proposed loss function, \textbf{GROUP-MITRE loss} ($G$-MITRE) $\Delta_{GM}({\bf y}_i, {\bf y})$, overcomes this limitation by adding, for each pedestrian described by the trajectory $T_i$, a fake counterpart $\alpha_{T_i}$ to which only singletons are connected. 
%Through this shrewdness we can now take into consideration singletons as well when computing the discrepancy between two solutions. The particular design choice to link to the fake counterparts only singleton members generates two discrepancies when committing errors involving singletons and is thus a further effort in generating hierarchical groups in the solution.

The complete algorithm for the computation of the $G$-MITRE loss is reported in Alg.~\ref{alg:G_MITRE}. We employ disjoint-set arrays due to the efficiency of checking whether two pedestrians belong to the same group. Recall that {\footnotesize UNION} and {\footnotesize FIND} are the standard functions defined over the disjoint-set arrays and denote the operations to merge two clusters and to find an element membership respectively. In the pseudo-code we use the notation ${\bf y}_i/{\bf y}$ to indicate that the algorithm first work on the solution ${\bf y}_i$ and then analogously on ${\bf y}$.

%To reduce the length of the pseudocode and avoid redundant lines due to the simmetrization process implicit in the loss computation, we use the notation ${\bf y}_i/{\bf y}$ to indicate that code is run twice, first the involved variable is ${\bf y}_i$ and then ${\bf y}$.

\begin{algorithm}[t!]
\setstretch{1.35}
\caption{$G$-MITRE loss $\Delta_{GM}({\bf y}_i, {\bf y})$ computation}
\label{alg:G_MITRE}
\begin{algorithmic}[1]
	\REQUIRE ${\bf y}_i$ and ${\bf y}$ as \emph{disjoint-set data structures}
%    \STATE Let $\varphi(\cdot) = \text{\footnotesize UNIQUE}(\cdot)$ be the connected components roots of the specified structure
%    \STATE Let $\Gamma(\cdot)$ be the size of the connected component whose root is specified
    \STATE $\varphi(x)$ are the unique roots of connected components $x$
    \STATE $\Gamma(x)$ is the size of the connected component with root $x$
	\FORALL{$T \in {\bf y}_i/{\bf y}$} 
		\STATE ${\bf y}_i/{\bf y}= {\bf y}_i/{\bf y}\cup\alpha_{T}$
		\IF{$\Gamma(\text{\footnotesize FIND}({\bf y}_i/{\bf y}(T)) = 1$}
			\STATE $\text{\footnotesize UNION}({\bf y}_i/{\bf y}(T), {\bf y}_i/{\bf y}(\alpha_{T}))$
		\ENDIF
	\ENDFOR	
	\FORALL{$q \in \varphi({\bf y}_i/{\bf y})$}
%		\STATE $\text{n}_{{\bf y}_i/{\bf y}}~{+=}~\Gamma(q) - |\varphi(\bigcup_{\text{\scriptsize FIND}({\bf y}_i/{\bf y}(T)) = q}{\bf y}/{\bf y}_i(T))|$
%		\STATE $\text{d}_{{\bf y}_i/{\bf y}}~{+=}~\Gamma(q) - 1$
		\STATE $v_{{\bf y}_i/{\bf y}}~{+=}~|\varphi(\bigcup_{\text{\scriptsize FIND}({\bf y}_i/{\bf y}(T)) = q}{\bf y}/{\bf y}_i(T))| - 1$
		\STATE $c_{{\bf y}_i/{\bf y}}~{+=}~\Gamma(q) - 1$
	\ENDFOR
	\STATE $\mathcal{R}_{{\bf y}_i/{\bf y}} = 1 - v_{{\bf y}_i/{\bf y}} / c_{{\bf y}_i/{\bf y}}$
	\STATE $\Delta({\bf y}_i, {\bf y}) = 1 - 2\mathcal{R}_{{\bf y}_i}\mathcal{R}_{{\bf y}} / (\mathcal{R}_{{\bf y}_i}+\mathcal{R}_{{\bf y}})$
\end{algorithmic}
\end{algorithm}

\subsection{Approximate Oracle}
\label{sec:oracle}
Despite the simplicity of the algorithm, the intrinsic complexity of the optimization is hidden in the search for the most violating solution ${\bf y}^*$ for the $i$-th example (line 4 of Alg.~\eqref{BCFW}): finding the most violated constraint requires to solve the loss augmented decoding subproblem. Note that the original prediction problem of Eq.~\eqref{eq:pred_fun} is NP-hard and the insertion of a non-linear loss in the computation of the maximum is not likely to help.
Nevertheless, thanks to its iterative nature, the inference scheme introduced in Sec.~\ref{sec:solution} can be adapted to approximate the oracle as well. Starting from the trivial solution having each pedestrian of the $i$-th example in its own cluster, the algorithm repeatedly merges the two clusters which reflect in the highest increment in the structured hinge-loss $\tilde{H}({\bf x}_i)$ of Eq.~\eqref{eq:maxoracle}, until a local maxima is found. 

%It's common practice to try to approximate the loss function $\Delta({\bf y}_i, {\bf y})$ in order to get a linear dependence on the solution ${\bf y}$ and formulate the loss augmented subproblem as a different instance of the original predictor, in such a way that the same algorithm can be employed to solve the prediction and the maximization oracle. Nevertheless, since we already used a greedy resolution scheme to solve the prediction problem (see Sec.~\ref{sec:solution}), we can take advantage of the iterative nature of the algorithm to preserve the non-linearity of the MITRE and $G$-MITRE loss, as suggested by Finley and Joachims~\cite{finley_supervised_2005}.
%Recall that the prediction algorithm works by initially considering all members isolated, at each iteration all potential mergings are then enumerated and the best one is chosen until the sum of pairs affinity inside each group cannot be further increased. At each iteration, for each possible solution, besides the computation of the scoring function we can now evaluate its respective loss too.

Of course by following a greedy procedure, there is no guarantee to select the most violated constraint. Interestingly enough, Lacoste-Julien~\emph{et al.}~\cite{julien_block_coordinate_2012} show that all convergence results known for exact maximizer of the loss augmented problem also hold for approximate maximizers by allowing the algorithm to iterate longer toward convergence.
% the number of convergence iterations may be are used instead, . What may change is the convergence speed which depends on the additive/multiplicative approximation amount.
For further details, please refer to their original work.
\begin{table*}[t!]
\centering
\caption{Comparative results on publicly available dataset using the $G$-MITRE loss of Sec.~\ref{sec:loss_score} and the positive pairwise loss $\Delta_{PW}^+$ of \cite{zanotto12}.}
%\begin{tabular}{|l|c|c||c|c||c|c||c|c|c|c|c|c|}
%\hline 
% & \multicolumn{2}{|c||}{our method ($\Delta_{GM}$)} & \multicolumn{2}{|c||}{our method ($\Delta_{PW}$)} & \multicolumn{2}{|c||}{baseline} & \multicolumn{2}{|c|}{\cite{ge_vision-based_2012}} & \multicolumn{2}{|c|}{\cite{yamaguchi_who_2011}} & \multicolumn{2}{|c|}{\cite{zanotto12}} \\
% & $\mathcal{P}$ & $\mathcal{R}$ & $\mathcal{P}$ & $\mathcal{R}$ & $\mathcal{P}$ & $\mathcal{R}$ & $\mathcal{P}$ & $\mathcal{R}$ & $\mathcal{P}$ & $\mathcal{R}$ & $\mathcal{P}$ & $\mathcal{R}$ \\ 
%\hline
%\emph{BIWI} \verb+hotel+ & \textls{\bf 97.3 $\pm$ 0.7} & \textls{\bf 97.7 $\pm$ 1.5} & & & \textls{71.0 $\pm$ 8.1} & \textls{69.6 $\pm$ 7.4} & \textls{86.9} & \textls{85.5} & \textls{91.3} & \textls{95.9} & \textls{87.0} & \textls{91.0}\\
%\hline
%\emph{BIWI} \verb+eth+ & \textls{\bf 91.8 $\pm$ 1.2} & \textls{\bf 94.2 $\pm$ 0.9} & & & \textls{72.4 $\pm$ 4.4} & \textls{65.2 $\pm$ 3.4} & \textls{87.0} & \textls{84.2} & \textls{83.0} & \textls{80.2} & \textls{79.0} & \textls{80.0} \\
%\hline 
%\emph{CBE} \verb+student003+ & \textls{\bf 81.7 $\pm$ 0.2} & \textls{\bf 82.5 $\pm$ 0.2} & & & \textls{59.9 $\pm$ 2.9} & \textls{53.5 $\pm$ 6.8} & \textls{77.2} & \textls{73.6} & \textls{80.5} & \textls{77.0} & \textls{70.0} & \textls{74.0}\\
%\hline
%\end{tabular}
\begin{tabular}{|lc|c|c||c|c||c|c|c|c|c|c|c|c|}
\hline 
 & &\multicolumn{2}{|c||}{our method} & \multicolumn{2}{|c||}{baseline} & \multicolumn{2}{|c|}{\cite{ge_vision-based_2012}} & \multicolumn{2}{|c|}{\cite{yamaguchi_who_2011}} & \multicolumn{2}{|c|}{\cite{zanotto12}} & \multicolumn{2}{|c|}{\cite{shao14}} \\
 & &  $\mathcal{P}$ & $\mathcal{R}$ & $\mathcal{P}$ & $\mathcal{R}$ & $\mathcal{P}$ & $\mathcal{R}$ & $\mathcal{P}$ & $\mathcal{R}$ & $\mathcal{P}$ & $\mathcal{R}$ & $\mathcal{P}$ & $\mathcal{R}$ \\ 
\hline
\rowcolor{LightGray}
\emph{BIWI}  & {$\Delta_{GM}$} & \textls{\bf 97.3 {\footnotesize $\pm$ 0.7}} & \textls{\bf 97.7 {\footnotesize $\pm$ 1.5}} & \textls{71.0 {\footnotesize $\pm$ 8.1}} & \textls{69.6 {\footnotesize $\pm$ 7.4}} & \textls{89.2} & \textls{90.9} & \textls{84.0} & \textls{51.2} & \textls{} & \textls{} & \textls{67.3} & \textls{64.1}\\
\verb+hotel+ & {$\Delta_{PW}^+$} & \textls{\bf 89.1 {\footnotesize $\pm$ 1.2}} & \textls{\bf 91.9 {\footnotesize $\pm$ 1.5}} & \textls{47.6 {\footnotesize $\pm$ 9.2}} & \textls{88.6 {\footnotesize $\pm$ 8.6}} & \textls{88.9} & \textls{89.3} & \textls{83.7} & \textls{93.9} & \textls{81.0} & \textls{91.0} & \textls{51.5} & \textls{90.4}\\
\hline
\rowcolor{LightGray}
\emph{BIWI}  & {$\Delta_{GM}$} & \textls{\bf 91.8 {\footnotesize $\pm$ 1.2}} & \textls{\bf 94.2 {\footnotesize $\pm$ 0.9}} & \textls{72.4 {\footnotesize $\pm$ 4.4}} & \textls{65.2 {\footnotesize $\pm$ 3.4}} & \textls{87.0} & \textls{84.2} & \textls{60.6} & \textls{76.4} & \textls{} & \textls{} & \textls{69.3} & \textls{68.2} \\
\verb+eth+ & {$\Delta_{PW}^+$} & \textls{\bf 91.1 {\footnotesize $\pm$ 0.4}} & \textls{\bf 83.4 {\footnotesize $\pm$ 0.6}} & \textls{39.1 {\footnotesize $\pm$ 8.4}} & \textls{91.2 {\footnotesize $\pm$ 1.7}} & \textls{80.7} & \textls{80.7} & \textls{72.9} & \textls{78.0} & \textls{79.0} & \textls{82.0} & \textls{44.5} & \textls{87.0} \\
\hline
\rowcolor{LightGray}
\emph{CBE} & {$\Delta_{GM}$} & \textls{\bf 81.7 {\footnotesize $\pm$ 0.2}} & \textls{\bf 82.5 {\footnotesize $\pm$ 0.2}} & \textls{59.9 {\footnotesize $\pm$ 2.9}} & \textls{53.5 {\footnotesize $\pm$ 6.8}} & \textls{77.2} & \textls{73.6} & \textls{56.7} & \textls{76.0} & \textls{} & \textls{} & \textls{40.4} & \textls{48.6} \\
\verb+student003+  & {$\Delta_{PW}^+$} & \textls{\bf  82.3 {\footnotesize $\pm$ 0.3} } & \textls{\bf  74.1 {\footnotesize $\pm$ 0.2}} & \textls{ 24.0 {\footnotesize $\pm$ 9.7}} & \textls{ 49.3 {\footnotesize $\pm$ 12.9}} & \textls{72.2} & \textls{65.1} & \textls{63.9} & \textls{72.6} & \textls{70.0} & \textls{74.0} & \textls{10.6} & \textls{76.0} \\
\hline
\end{tabular}
\label{tab:comparison}
\vspace{-0.5cm}
\end{table*}

\section{Experimental Results}

\label{sec:exp}
We designed several experiments to evaluate the algorithm behavior on well-assessed benchmarks and its connections to the nature of the problem.
All the experiments were carried out on ground truth trajectory data, except for Sec.~\ref{exp:real} where the method is evaluated on tracklets extracted by a modern detector/tracker system.
We also propose new video sequences to stress the algorithm over a variety of challenges in real world scenarios. Since the method works on ground plane (metric) data, we also provide homography information for all the employed sequences.
%To this end, we organized the experimental section as follows. First, we briefly characterize the datasets employed in the testing procedure and present some comparison result on well-assessed benchmarks. In the second part we move on to investigate the peculiar properties of our algorithm and stress it over a variety of challenges in real world scenarios.
%In order to extensively evaluate our proposal we specifically designed several experiments aimed above all at gaining new insights on the algorithm behavior and its connections to the nature of the problem itself.
%To this end, we organized the experimental section as follows. First, we briefly characterize the datasets employed in the testing procedure and present some comparison result on well-assessed benchmarks. In the second part we move on to investigate the peculiar properties of our algorithm and stress it over a variety of challenges in real world scenarios.
%Lastly, we consider a modification of our method able to account for groups shape as well and compare the new results with our original achievements.

\subsubsection*{Datasets}
We selected two publicly available datasets, namely the \emph{BIWI Walking Pedestrians} dataset~\cite{pellegrini_youll_2009} and the \emph{Crowds-By-Examples (CBE)} dataset~\cite{lerner_crowds_2007}. The former dataset records two low crowded scenes, outside a university and at a bus stop (\verb+eth+ and \verb+hotel+ in Tab.~\ref{tab:dataset}). The \emph{CBE} dataset records a medium density crowd outside another university (\verb+student003+, briefly \verb+stu003+) providing some challenges: the density of the pedestrians is significantly high and the presence of multiple entry and exit points. While \emph{BIWI} and \emph{CBE} are standard datasets in crowd analysis, we also use the more recent \emph{Vittorio Emanuele II Gallery (VEIIG)} dataset~\cite{BanGorVizPRLGallery}, from which we extracted a five minutes subsequence, \verb+gal1+, particularly interesting due to the fast and continuous change in crowd density.
We also propose a new dataset to cope with the increasing variety of application in dense-crowd management, \emph{MPT-$20$x$100$}, composed of 20 sequences of 100 frames where we manually annotated trajectories and social groups. The dataset comprises different videos~\cite{bolei_2014} all characterized by a high number of pedestrians with an heterogeneous set of scene conditions, ranging from density, scale, viewpoint and type of interactions, like walking in a mall, crossing the street or participating at public events.\\
In Tab.~\ref{tab:dataset} we report some measures useful to characterize the spatial complexity of the datasets:
%An overview of the datasets employed is presented in Tab.~\ref{tab:dataset}, where the number of pedestrians, the number of groups and a set of quantitative measures of people closeness are introduced to gain insight on the density and the complexity of these sequences:
\begin{itemize}
	\item $d_\text{in}$ is the \emph{group compactness}, computed as the mean distance between members of the same groups;
	\item $d_\text{out}$ is the \emph{group isolation} or the mean distance between each member and its closest unrelated pedestrian;
	\item the ratio $d_\text{i/o}\stackrel{\text{\tiny def}}{=}d_\text{in}/d_\text{out}$ measures \emph{crowd collectiveness}: small values mean compact groups in a sparse crowd.
\end{itemize}

\begin{table}[t!]
\center
\caption{Datasets: number of pedestrians (\#p), groups (\#g) and density metrics.}
% Data on \emph{MPT-20x100} is averaged. }
\begin{tabular}{|l|c|c|c|c|c|c|}
\hline 
 & \#p & \#g & $d_\text{in}~(m)$ & $d_\text{out}~(m) $ & $d_\text{i/o}$ \\
\hline
%\multicolumn{6}{|l|}{CBE}\\
%\hline
\verb+student003+ & 406 & 108 & 0.41 & 0.70 & 0.59\\
\hline
%\multicolumn{6}{|l|}{BIWI}\\
%\hline
\verb+eth+ & 117 & 18 & 0.99 & 2.79 & 0.35\\ 
\hline 
\verb+hotel+ & 107 & 11 & 0.75 & 2.00 & 0.38\\ 
\hline
%\multicolumn{6}{|l|}{VEIIG}\\
%\hline 
\verb+gal1+ & 630 & 207 & 0.77 & 1.66 & 0.46\\
\hline
%\multicolumn{6}{|l|}{MPT-20x100 (averaged)}\\
%\hline
%\verb+20 sequences+ & 82 & 10 & 0.63 & 1.45 & 0.48\\
\verb+MPT-20x100+ & 82 & 10 & 0.63 & 1.45 & 0.48\\
\hline
%\verb+1airport1+ & 73 & 11 & 0.93 & 1.45 & 0.64\\ 
%\hline
%\verb+1chinacross2+ & 51 & 6 & 1.14 & 3.75 & 0.30\\ 
%\hline
%\verb+1chinacross4+ & 56 & 4 & 0.89 & 3.92 & 0.23\\ 
%\hline
%\verb+1dawei1+ & 67 & 5 & 0.59 & 1.64 & 0.36\\ 
%\hline
%\verb+1dawei5+ & 53 & 4 & 0.78 & 1.85 & 0.42\\
%\hline 
%\verb+1grand1+ & 70 & 5 & 0.54 & 1.36 & 0.40\\
%\hline 
%\verb+1grand3+ & 81 & 7 & 0.67 & 1.47 & 0.46\\
%\hline
%\verb+1japancross2+ & 148 & 18 & 0.40 & 0.91 & 0.44\\
%\hline 
%\verb+1japancross3+ & 108 & 14 & 0.39 & 0.98 & 0.40\\
%\hline
%\verb+1manko3+ & 74 & 19 & 0.59 & 0.99 & 0.60\\
%\hline
%\verb+1manko29+ & 31 & 6 & 0.53 & 1.23 & 0.43\\
%\hline
%\verb+1shatian3+ & 239 & 47 & 0.64 & 1.34 & 0.48\\
%\hline
%\verb+1thu10+ & 35 & 2 & 0.41 & 0.7 & 0.59\\
%\hline
%\verb+2dawei1+ & 90 & 3 & 0.38 & 0.83 & 0.46\\ 
%\hline
%\verb+2grand6+ & 113 & 4 & 0.59 & 1.36 & 0.43\\
%\hline
%\verb+2jiansha5+ & 66 & 13 & 0.73 & 1.13 & 0.65\\
%\hline
%\verb+2manko2+ & 88 & 8 & 0.6 & 1.02 & 0.59\\
%\hline
%\verb+2niurunning2+ & 82 & 10 & 0.76 & 0.96 & 0.79\\
%\hline
%\verb+3shatian6+ & 60 & 10 & 0.50 & 1.29 & 0.39 \\
%\hline
%\verb+randomcross3+ & 57 & 0 & - & 0.82 & - \\
%\hline
\end{tabular} 
\label{tab:dataset}
\end{table}

\subsubsection*{Evaluation Scheme}
There is no consensus on which metrics should be used to evaluate groups correctness: we propose to use the $G$-MITRE precision $\mathcal{P}$ and recall $\mathcal{R}$ since it accounts for the correct classification of singletons as well. This is an important gain as in crowded scenes the number of people walking alone is rarely negligible.
Each measure is reported in terms of mean and standard deviation over 5 runs to account for the stochastic nature of the training of our algorithm. Where not differently specified, we used a 100s for training and a 10s sliding window with no overlap for features computation. The regularization parameter $C$ of QP.~\eqref{optpro} is fixed to 10.\\

\noindent For the heat-map based feature of Sec.~\ref{sec:heatmaps}, we run a grid search on the parameters. For all the experiments, the length of the cells edge is fixed to 30cm, $k_s=10^{-5}$ and $k_r=0.5$.

\subsection{Baseline and Benchmark Comparisons}
We compare our method with three recent state of the art group detection algorithms, namely~\cite{ge_vision-based_2012,yamaguchi_who_2011,zanotto12,shao14}, selected on the basis of their reported performances on public datasets and availability of code.
In addition, we devised a simple baseline version of our solution that performs the group partitioning with no use of the learning framework. The weights are randomly chosen to be the same for all the features, so that the randomness resides in the similarity/dissimilarity ratio.

\begin{table}[t!]
\centering
\caption{Evaluation of our proposal when trained with different loss functions.}
% The $G$-MITRE is used for evaluation.}
\begin{tabular}{|l|c|c|c|c|}
\hline 
 & \multicolumn{2}{|c|}{Pairwise $\Delta_{PW}$} & \multicolumn{2}{|c|}{MITRE $\Delta_{M}$} \\
 & $\mathcal{P}$ & $\mathcal{R}$ & $\mathcal{P}$ & $\mathcal{R}$ \\ 
\hline
\verb+hotel+ & \textls{90.1 $\pm$ 2.0} & \textls{84.1 $\pm$ 3.2} & \textls{89.2 $\pm$ 3.0}& \textls{93.2 $\pm$ 1.9} \\
\hline
\verb+eth+ & \textls{88.7 $\pm$ 1.8} & \textls{87.3 $\pm$ 2.6} & \textls{91.9 $\pm$ 0.8}& \textls{92.9 $\pm$ 1.0}\\
\hline 
\verb+stu003+ & \textls{68.9 $\pm$ 1.4} & \textls{69.9 $\pm$ 1.5} & \textls{80.1 $\pm$ 2.4}& \textls{80.9 $\pm$ 2.3}\\
\hline
\end{tabular}
\label{tab:loss}
\end{table}

\subsubsection{Quantitative Results}
Quantitative results are given in Tab. \ref{tab:comparison}. To highlight our algorithm superiority, results are presented both in terms of $G$-MITRE and a pairwise loss accounting only for positive (intra-group) relations  but neglecting singletons, $\Delta_{PW}^+$~\cite{zanotto12}. The latter loss is not directly optimized by our algorithm, still our method outperforms the competitors in all the tested sequences. This can be explained through the ability of our algorithm to adapt the concept of groups to always different scenario by varying the feature importance and the use of sociologically inspired similarity functions. The slightly lower performances on the \verb+stu003+ sequence are due to the high complexity of the scene: the high value of the $d_\text{i/o}$ ratio in Tab. \ref{tab:dataset} suggests the presence of loose groups in a dense crowd and, as such, challenging to be detected.

\subsubsection{Evaluation of Different Loss Functions}
As structured learning relies upon a definition of \emph{what's wrong} to learn how to classify well, the choice of the loss function can greatly affect the final performances. By fixing the $G$-MITRE measure as a proper scoring scheme, we quantitatively test the influence of the choice of the loss on the \verb+eth+, \verb+hotel+ and \verb+stu003+ datasets (Tab.~\ref{tab:loss}).
As it could be expected by its definition, the improvement due to the use of the $G$-MITRE loss (reported in Tab.~\ref{tab:comparison}) is greater in the \verb+eth+ and \verb+hotel+ sequences where the ratio between the number of singletons and the people walking in groups is higher and as such learning to classify them as well becomes crucial. More interestingly, we observe how the pairwise loss obtains outstanding performances when the number of pedestrians is limited, but becomes ineffective when it starts to grow, as in \verb+stu003+.

%\subsubsection{How Training Data Influences Convergence?}
%Since our method relies on training data, we evaluate the performance varying the training set size. We used for training data video sequences of increasing length. Computed F-1 score values of the algorithm are depicted in Fig. \ref{fig:traininglen}. The plot highlights that once the learning algorithm reaches convergence, performance stability is fairly achieved. Moreover, observing the type of considered sequences it is evident that those with very few group examples (i.e. \verb+eth+, \verb+hotel+) benefit the most from extending the training time, as opposite to the \verb+stu003+ sequence which has a high number of groups already in the beginning of the video leading to a faster convergence of the weights learning process. We can conclude that the learning speed depends on the number of groups observed by the algorithm and not on the time-length of the training sequence and, once convergence is reached, the performances stay stable even if the density of the crowd mildly changes (e.g. in the \verb+stu003+ sequence), suggesting the learner exhibits good generalization ability. As a consequence, we empirically choose to set the standard training set size to 100s as this is the sufficient amount of time needed for all videos to converge.
%
%\begin{figure}[tbh]
%  \centering
%    \includegraphics[width=\columnwidth]{images/training_set_size.pdf}
%    \caption{Performance of our method when trained with sequences of different length. Shaded area depicts the variance of the results.}
%    \label{fig:traininglen}
%\end{figure}

\subsection{Features Weight Learning on \emph{MPT-${\bf 20}$x${\bf 100}$}}
\emph{CBE} and \emph{BIWI} datasets expose some interesting challenges of the problem but, with the only exception of \verb+stu003+ sequence, they have a limited number of pedestrians in scene and a low crowd density. Moreover, the scenarios are similar and the variety of interactions underlying the group formation is limited. The proposed \emph{MPT-$20$x$100$} datasets, on the other hand, presents different degrees of complexity.
%, useful to stress a group detection algorithm.

%Publicly available datasets for crowd partitioning expose some interesting challenges of the problem but, with the only exception of \verb+student003+ sequence, they have a limited number of pedestrians in scene and a low crowd density. Moreover, the scenario type is always the same and consequently the variety of interactions underlying the group formation is also limited. For these motivations in the following experiment we propose a new dataset that represent many real world surveillance scenarios with different degrees of difficulty, useful to stress a group detection algorithm.
%\begin{figure}[tbh]
%	\centering
%	\subfloat{
%		\includegraphics[width=0.3\columnwidth]{images/dataset/000001}
%    }
%	\subfloat{
%		\includegraphics[width=0.3\columnwidth]{images/dataset/000008}
%    }
%	\subfloat{
%		\includegraphics[width=0.3\columnwidth]{images/dataset/000012}
%    }
%	\caption{Examples of video sequences in the \emph{MPT-$20$x$100$} dataset.}
%	\label{fig:MPT}
%\end{figure}

\begin{figure}[t!]
   \centering
    	\includegraphics[width=\columnwidth]{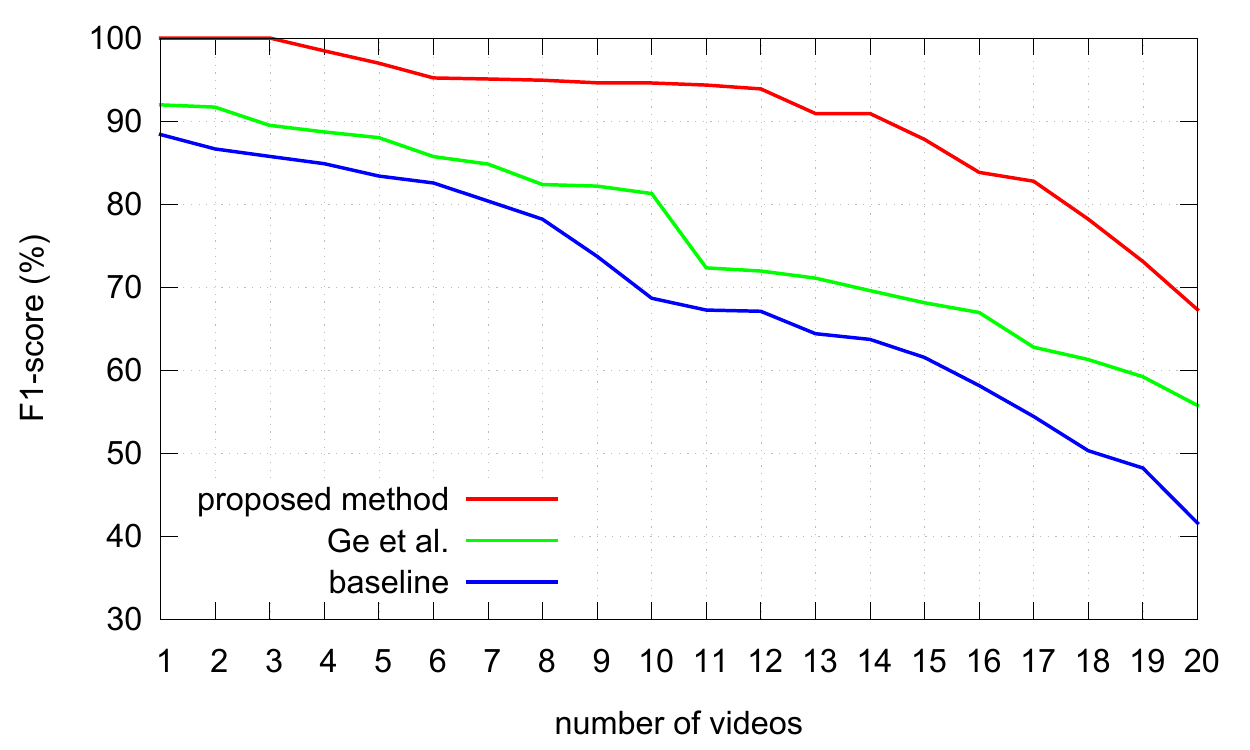}
   \caption{Comparison against baseline and \cite{ge_vision-based_2012} on \emph{MPT-$20$x$100$}.}
   \label{fig:perf_comp_b}
\end{figure}
First, we evaluate the general performance of the algorithm and compare with both our baseline and the proposal in \cite{ge_vision-based_2012} where, for the latter method, we manually tuned the thresholds to achieve best results.
%\footnote{We can not compare with \cite{yamaguchi_who_2011} and \cite{zanotto12} here as the code was not available.}.
These methods are clustering based, partially consistent with the social group axioms but no learning is employed.
%On the other hand we trained our algorithm on each of these sequences independently; due to the short length of the videos both training and testing sets cover the same whole 100 frames which are analyzed all in the same time window.
Results are shown in Fig.~\ref{fig:perf_comp_b} as a \emph{survival curve} plot which reveals on how many sequences the algorithms where at least able to reach the specific lower-bound performance and per-video scores are in Fig.~\ref{fig:perf_comp_a}. Interestingly, the difference between our method and \cite{ge_vision-based_2012} increases here with respect to the previous datasets on an average of 10\%, suggesting that sequences can be really different in the concept of groups they embed and thus learning is mandatory to adapt to this new representations of social groups and keep performances stable.

\subsubsection{The Need for Learning from Examples}
\begin{figure}[t!]
   \centering
    	\includegraphics[width=\columnwidth]{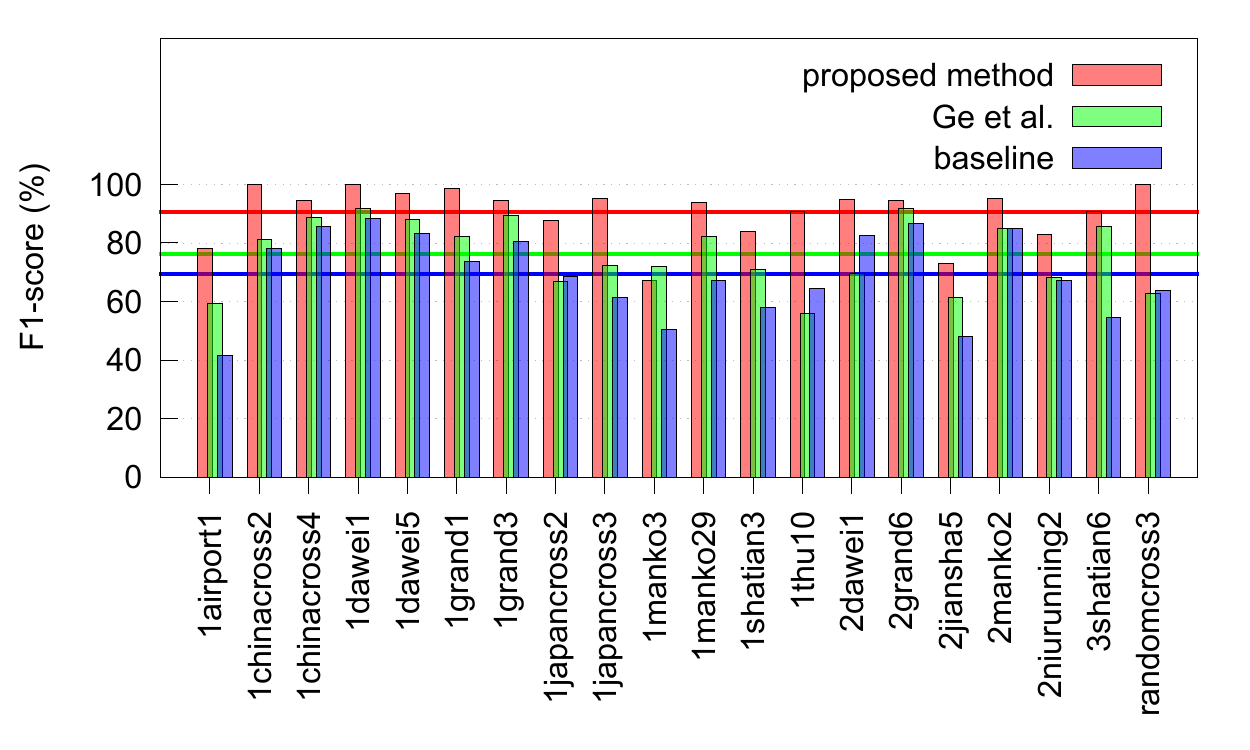}
   \caption{Results on \emph{MPT-$20$x$100$} highlight the complexity of each scene.}
   \label{fig:perf_comp_a}
\end{figure}
The confusion-like matrix, depicted in Fig. \ref{fig:matrix}, presents the F-1 scores obtained by training the algorithm on one sequence of \emph{MPT-$20$x$100$} (row labels) and testing it on all the other sequences (column labels).
By reading the matrix, and averaging each row over all the columns, it is possible to grasp how good a particular sequence was for training. At the same time, by observing the average of the columns over all the rows, we can get intuition about how much each sequence was effectively predicted by all the others.\\
We are interested in understanding whether a specific notion of group is shared across sequences and how it is influenced by both scene elements (\emph{e.g.} crowd density) and unobserved aspects (\emph{e.g.} intentions and social hierarchies).
%
%In order to formally verify the need and the benefits of a learning algorithm to approach the group detection task, we built a confusion-like matrix, depicted in Fig. \ref{fig:matrix}, where we present the F-1 scores obtained by training the algorithm on one sequence of \emph{MPT-$20$x$100$} (row labels) and testing it on all the other sequences (column labels).
%By reading the matrix, and averaging each row over all the columns, it is possible to grasp how good a particular sequence was for training. At the same time, by observing the average of the columns over all the rows, we can get intuition about how much each sequence was effectively predicted by all the others.\\
%Nonetheless, we do not insert into the model any prior assumptions on the structure of the dataset, rather we are interested in understanding if there is a specific notion of group that is shared across some sequences and scenarios and how it is influenced by both scene elements (\emph{e.g.} crowd density) and unobserved aspects (\emph{e.g.} intentions and social hierarchies).
\begin{figure}[t!]
  \centering
    \includegraphics[width=\columnwidth]{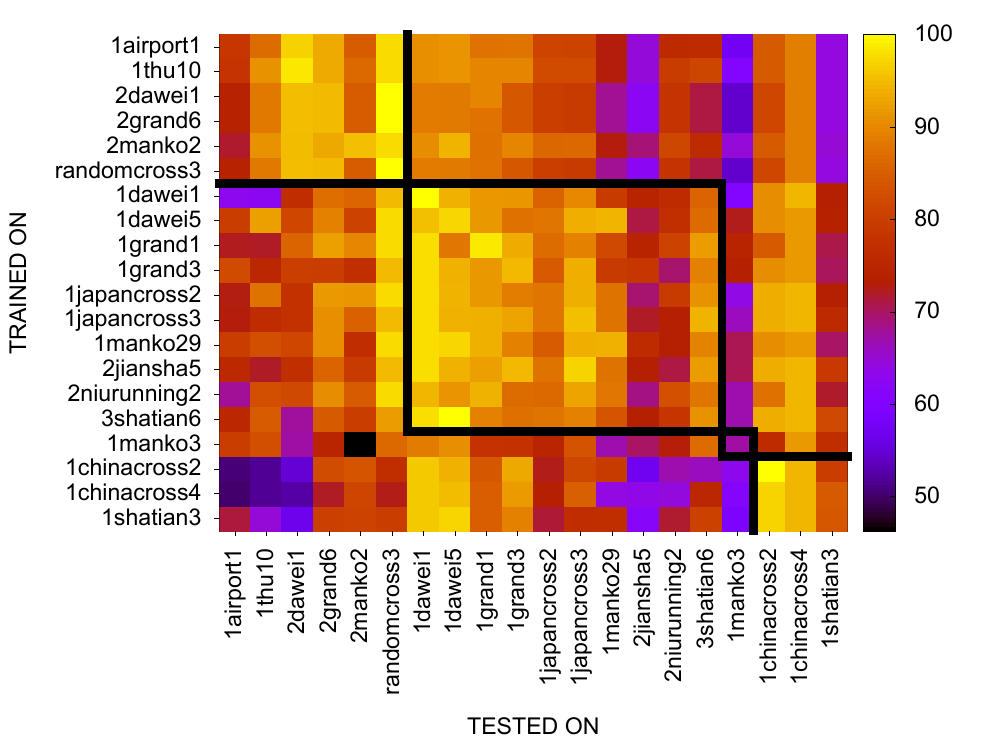}
  \caption{F-1 scores obtained by all combinations of train/test pair sequences in \emph{MPT-$20$x$100$}. Results were clustered (diagonal blocks C1-C4 from left to right) to highlight similar notion of group among sequences.}
  \label{fig:matrix}
\end{figure}

With the purpose of capturing these invariants, we search the connected component of the matrix using the $\text{F-1 score}$ as the affinity value among elements. Clustering is performed through an asymmetric version of spectral clustering~\cite{naumann_combinatorial_2012} based on the Random Walk Laplacian defined as
\begin{equation}
L = AD^{-1},
\end{equation}
where $A$ is the affinity matrix defined as in Fig.~\ref{fig:matrix} and $D$ is the usual degree matrix. Following the eigen-gap heuristic we found $4$ distinct clusters in the \emph{MPT-$20$x$100$} dataset, highlighted with black lines in Fig. \ref{fig:matrix}; for every cluster we computed the $d_\text{in}$, $d_\text{out}$ and $d_\text{i/o}$ spatial measures, displayed in Tab. \ref{tab:spatial}, to verify if clusters with a similar notion of group also share a common configuration of distances among pedestrians and possibly if the performance are connected to crowd density.

\begin{table}[t!]
\centering
\caption{Spatial depiction, training efficacy and groups predictability of the clusters of sequences of Fig.~\ref{fig:matrix}.}
\begin{tabular}{|l|c|c|c|c|c|}
\hline 
cluster & $d_\text{in}$ (m)& $d_\text{out}$ (m)& $d_\text{i/o}$ & $\text{F}_1$ train & $\text{F}_1$ test\\
\hline
\hline
C1 & 0.58 & 1.03 & 0.54 & 0.82 & 0.82 \\
\hline
C2 & 0.59 & 1.28 & 0.47 & 0.85 & 0.84 \\
\hline
C3 & 0.59 & 0.99 & 0.59 & 0.77 & 0.64 \\
\hline
C4 & 0.89 & 3.00 & 0.34 & 0.75 & 0.84 \\
\hline
\end{tabular}
\label{tab:spatial}
\end{table}

%\begin{table*}[bth]
%\centering
%\caption{Spatial depiction, training efficacy and groups predictability of the clusters of sequences obtained from Fig.~\ref{fig:matrix}.}
%\begin{tabular}{|l|c|c|c|c|c|}
%\hline 
%clusters & $d_\text{in}$ (m)& $d_\text{out}$ (m)& $d_\text{i/o}$ & mean F-1 train & mean F-1 test\\
%\hline
%\hline
%1airport1, 1thu10, 2dawei1, 2grand6, & & & & &\\
%2manko2, randomcross3 & 0.58$\pm$0.22 & 1.03$\pm$0.31 & 0.54$\pm$0.09& 0.817& 0.821\\
%\hline
%1dawei1, 1dawei5, 1grand1, 1grand3,& & & & &\\
%1japancross2, 1japancross3, 1manko29,& 0.59$\pm$0.14 & 1.28$\pm$0.31 & 0.47$\pm$0.14& 0.851& 0.839\\
%2jiansha5, 2niurunning2, 3shatian6
% & & & & &\\
%\hline
%1manko3 & 0.59$\pm$0.00 & 0.99$\pm$0.00 & 0.59$\pm$0.00& 0.767& 0.641\\
%\hline
%1chinacross2, 1chinacross4, 1shatian3& 0.89$\pm$0.25 & 3.00$\pm$1.44 & 0.34$\pm$0.13& 0.747& 0.844\\
%\hline
%\end{tabular}
%\label{tab:spatial}
%\end{table*}

Tab. \ref{tab:spatial} also reports a measure of \emph{training efficacy} (F$_1$ train), computed as the mean accuracy obtained on the whole dataset when only sequences in that specifi cluster were used for training and, analogously, a group \emph{predictability score} (F$_1$ test) or the mean accuracy obtained on the sequences of that cluster when all the sequences were used for training. They indicate how much a cluster is useful during  training
%, \emph{i.e.} when it exhibits a high training efficacy or, on the other hand,
and easy it is to predict groups inside its sequences.
%, \emph{i.e.} when its associated average testing accuracy is high.

A first observation that can be made is about the cluster C4, which presents the highest F$_1$ test and the lowest F$_1$ train. We found it was easy to predict groups in these videos but they were poorly informative as training examples, a result justified by its small $d_\text{i/o}$.
%, asserting the sequences of this cluster are easy to solve, but very poorly informative as training examples.
% This is motivated by a particularly small $d_\text{i/o}$ ratio meaning that the groups are strongly separated by the rest of the people in the scene.
%Nevertheless, clusters A and C characterized by a similar $d_\text{i/o}$ ratio perform very differently in terms both of training efficacy and testing score.
% The two best compromises between training efficacy and groups predictability are the clusters with a $d_\text{i/o}$ ratio which lays in the middle supporting the importance of relative distances and group compactness theorized by Hall in his proxemics theory.
Nonetheless, clusters 1 and 3 exhibits very similar $d_\text{i/o}$ ratio but perform very differently in terms both of training efficacy and testing score, suggesting a trivial heuristic based on spatial information only is insufficient to visually discern groups.
% These observations conclude that there exists a relation among crowd density, people mutual distances and the capability of visually detecting groups but, at the same time, a trivial heuristic devised only on spatial information cannot be employed.
Implicit aspects like motion constraints or cultural and social context also affect the group process formation, defending our hypothesis that learning is needed to adapt the concept of group to the current data.
%These observations conclude that there exists a relation among crowd density, people mutual distances and the capability of visually detecting groups but, at the same time, a trivial heuristic devised only on spatial information cannot be employed. As a matter of fact, implicit aspects such as the specific scenario and the cultural or social context affect the group process formation, defending our hypothesis that learning needs to be performed in order to adapt the concept of group to the current data.\\

\subsubsection{Do we Capture the Essence of Being a Group?}
%As described in Sec. \ref{sec:features}, the definition of social groups involves both spatial and social characteristics, which we have been taken into account when we devised the adopted features. In order to apprehend the importance of the chosen features, we compare the learned feature weights between sequences of the heterogeneous \emph{MPT-$20$x$100$} dataset.
As previously stated, \emph{MPT-$20$x$100$} comprises very different scenarios and situations and can provide important insights on which are the most important elements that reveal groups. To this end, recall the definition of feature vector ${\bf w} = [\boldsymbol\alpha, \boldsymbol\beta] = [w_1,w_2,\dots,w_8]$ from Eq.~\eqref{eq:cc_affinity_parametrization} of Sec.~\ref{sec:solution} is such that the affinity between two trajectories $T_a$ and $T_b$ can be written as:
%\begin{equation}
%\label{eq:w_decomposed}
%\begin{aligned}
%&W^{ab}_{\bf d} = {\boldsymbol\alpha}^T ({\bf 1} - {\bf d}(a, b)) - {\boldsymbol\beta}^T {\bf d}(a, b)\\
%%=& w_1(1-d_{ph})+w_2(1-d_{sh})+w_3(1-d_{ca})+w_4(1-d_{he}) - w_5d_{ph} - w_6d_{sh} - w_7d_{ca} - w_8d_{he}\\
%&= \underbrace{w_1 + w_2 + w_3 + w_4}_\text{constant term} - \dots\\
%& \underbrace{[(w_5+w_1)d_{ph} + (w_6+w_2)d_{sh} + (w_7+w_3)d_{ca} + (w_8+w_4)d_{he}]}_\text{$(a,b)$-dependent term}
%\end{aligned}
%\end{equation}
%
%\begin{equation}
%\label{eq:w_decomposed}
%\begin{aligned}
%W^{ab}_{\bf d} = {\boldsymbol\alpha}^T ({\bf 1} - {\bf d}(a, b)) - {\boldsymbol\beta}^T {\bf d}(a, b)\\
%\begin{rcases}
%= & w_1 + w_2 + w_3 + w_4 - \dots
%\end{rcases}\\
%\begin{rcases}
%=& [(w_5+w_1)d_{ph} + (w_6+w_2)d_{sh} + \dots \\
%&(w_7+w_3)d_{ca} + (w_8+w_4)d_{he}]
%\end{rcases}
%\end{aligned}
%\end{equation}
%
\begin{equation}
\label{eq:w_decomposed}
\begin{aligned}
&W^{ab}_{\bf d} &= {\boldsymbol\alpha}^T ({\bf 1} - {\bf d}(a, b)) - {\boldsymbol\beta}^T&{\bf d}(a, b)\\
&&= w_1 + w_2 + w_3 + w_4 - &[(w_5+w_1)d_{ph} + \dots\\
&& &~(w_6+w_2)d_{sh} + \dots\\
&& &~(w_7+w_3)d_{ca} + \dots\\
&&\underbrace{\hspace{3cm}\vphantom{]}}_\text{constant term}~&\underbrace{(w_8+w_4)d_{he}}_\text{$(a,b)$-dependent term}]
\end{aligned}
\end{equation}
%
%It is easy to see how the contribution of each feature to the score, transformed from a distance to an affinity measure by the constant term of Eq.~\eqref{eq:w_decomposed}, is encoded in the absolute value of the coefficient of the features themselves. In order to compare them, we normalize the absolute value coefficients as shown in Fig.~\ref{fig:weight_imp_a}.
The contribution of each feature to the score, transformed from a distance to an affinity measure by the constant term of Eq.~\eqref{eq:w_decomposed}, is encoded in the absolute value of the coefficient of the features themselves.

%\begin{figure}[tbh]
%  \centering
%  	\subfloat[]{
%    	\includegraphics[width=0.3\columnwidth]{images/1dawei1_000015.jpg}
%	}
% 	\subfloat[]{
%    	\includegraphics[width=0.3\columnwidth]{images/1manko3_000025.jpg}
%  	}
%    \subfloat[]{
%    	\includegraphics[width=0.3\columnwidth]{images/3shatian6_000100.jpg}
%  	}
%  \label{fig:weight_seq}
%  \cprotect\caption{\verb+1dawei1+, \verb+1manko3+ and \verb+3shatian6+ are examples of sequences where the learned features weights gain different relative importance.}
%\end{figure}

As shown in Fig.~\ref{fig:weight_imp_a}, the proxemic inspired feature $d_{ph}$ dominates all the others while the importance of the remaining features vary greatly from sequence to sequence.
The two sequences \verb+1manko3+ (Fig.~\ref{fig:more_results}) and \verb+1dawei1+ (Fig.~\ref{fig:crowd}), for example, present very similar contribution from $d_\text{hm}$ and $d_\text{sh}$, while the importance assigned to $d_\text{ph}$ in \verb+1dawei1+ is shifted to $d_\text{ca}$ in \verb+1manko3+.
%By taking a deeper look at the sequences in Fig.~\ref{fig:weight_seq}, it is possible to see how
The former sequence present a particularly sparse crowd, making distance among elements a strong peculiarity of groups, but when the space among pedestrian is reduced both intra and inter-groups distances (and consequently $d_\text{ph}$) become less significant. Conversely, the causality feature $d_{ca}$ becomes more important when the density increases as pedestrians tend to follow each others to avoid getting separated from the rest of the group.
\begin{figure}[t!]
  \centering
    \includegraphics[width=\columnwidth]{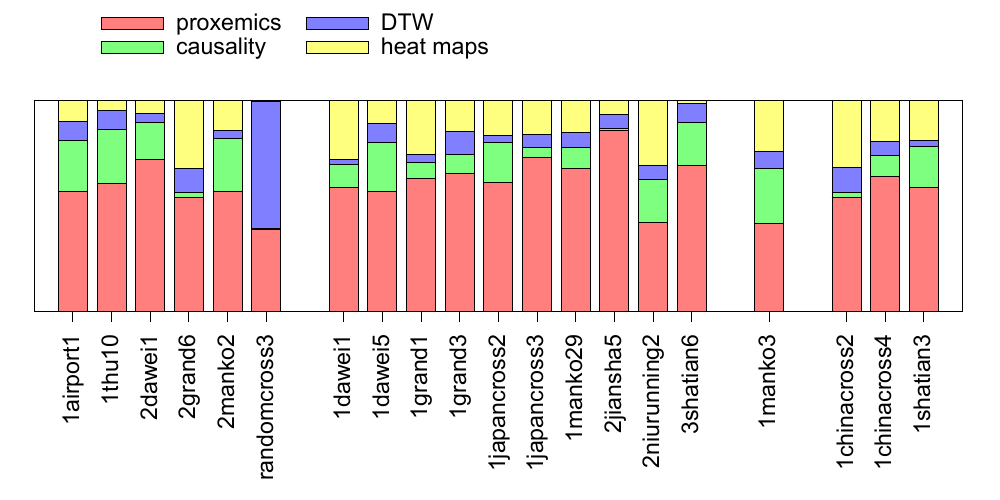}
  \caption{Features normalized coefficients of Eq.~\eqref{eq:w_decomposed}.}
  \label{fig:weight_imp_a}
\end{figure}
Heat maps importance gain emphasis from comparing \verb+1manko3+ and \verb+3shatian6+ (Fig.~\ref{fig:more_results}), as they are very helpful in decoupling trajectories that stand very close in space but for a very limited amount of time. In particular, in \verb+1manko3+, people crossing from opposite sides of the road tend to be very close when meeting in the middle, even if they are not in the same group. %but it is very unlikely they belong to the same group.
%Heat maps importance gain emphasis from comparing the weights obtained for sequences \verb+1manko3+ and \verb+3shatian6+: in this case both causality and DTW importance are constant, while heat maps become more important in \verb+1manko3+ and almost disappear in \verb+3shatian6+. The main reason behind the success of heat maps is that they are very helpful in decoupling trajectories that stand very close in space but for a very limited amount of time. In particular, in \verb+1manko3+, people crossing from opposite sides of the road tend to be very close when meeting in the middle, but it is very unlikely they belong to the same group. Heat maps help to notice that there is no temporal prolongation of this condition and thus weight the pair of trajectories accordingly.
%We additionally averaged the feature weights over the clusters obtained in the previous experiment (see Fig.~\ref{fig:matrix}) and we plot them in Fig. \ref{fig:weight_imp_b}.
%In the figure it can be observed how all the clusters agree on the importance of the proxemic feature $d_{ph}$ but, at the same time, every cluster privileges a set of additional features that capture the peculiarities of the observed groups, \emph{e.g.} the fourth cluster privileges heat maps feature while in the first cluster both causality and shape features gain importance in group characterization.

%\begin{figure}
%	\centering
%	\includegraphics[width=0.5\columnwidth]{images/features_mean_weights.pdf}
%	\caption{Average per-cluster features weights importance}
%	\label{fig:weight_avg}
%\end{figure}

\begin{table*}[t!]
\centering
\caption{Performance of detector \cite{DollarPAMI14pyramids}, tracker \cite{Milan:2014:CEM} and group detection algorithms (in terms of G-MITRE) in a fully automatic pipeline.}
\begin{tabular}{|l|c|c||c|c|c|c||c|c|c|c|c|c|c|c|}
	\hline
	$ $ & \multicolumn{2}{|c||}{Detector} & \multicolumn{4}{|c||}{Tracker} & \multicolumn{2}{|c|}{our proposal} & \multicolumn{2}{|c|}{\cite{ge_vision-based_2012}} & \multicolumn{2}{|c|}{\cite{yamaguchi_who_2011}} & \multicolumn{2}{|c|}{\cite{shao14}}\\
	$ $ & $\mathcal{P}$ & $\mathcal{R}$ & MOT(A/P) & MT & IDS & FRG & $\mathcal{P}$ & $\mathcal{R}$ & $\mathcal{P}$ & $\mathcal{R}$ & $\mathcal{P}$ & $\mathcal{R}$ & $\mathcal{P}$ & $\mathcal{R}$\\
	\hline
	\verb+hotel+ & 43.1 & 52.4 & 66.9 / 0.88 & 18.8 & 120 & 34 & 77.9 & 76.9 & 75.7 & 78.0 & 46.3 & 38.6 & 60.2 & 57.5\\
	\verb+eth+ & 68.2 & 53.7 & 92.3 / 0.08 & 75.0 & 0 & 68 & 81.1 & 79.7 & 78.4 & 79.3 & 58.3 & 70.6 & 57.3 & 61.2\\
	\verb+student+ & 56.7 & 36.8 & 43.3 / 1.22 & 06.0 & 342 & 876 & 75.0 & 71.3 & 63.2 & 56.4 & 40.2 & 52.4 & 35.1 & 40.2\\
	\hline
\end{tabular}
\label{tab:track}
\end{table*}

\subsection{Evaluating the Influence of Density Changes}% on the \emph{GVEII} Dataset}
%
%\begin{figure*}[t!]
%  \centering
%  %\sidesubfloat[]{
%  \subfloat[]{
%    \includegraphics[width=\columnwidth]{images/features_sequential.pdf}
%	}
%  %\sidesubfloat[]{
%  \subfloat[]{
%      \includegraphics[width=\columnwidth]{images/features_online.pdf}
%	}
%  \cprotect\caption{Features weights learned during (a) sequential and (b) online training on \verb+gal1+ of \emph{GVEII}. Batch learned weights can be observed in the leftmost 
%part of plot (b) since they were used to initialize the online process.}
%  \label{fig:GVEIIweights}
%\end{figure*}
%
In this test setting we evaluate if the feature weights learned by the Structural SVM of Sec. \ref{sec:learning} are sufficiently general to deal with crowds at different densities and, at the same time, understand whether an online version of Alg.~\ref{BCFW} would bring any accuracy improvement. To this end we introduce a new video sequence, \verb+gal1+ from \emph{GVEII}, containing an average number of $70$ pedestrians simultaneously present in the scene.
The distribution of pedestrians is not uniform though, and increases over time, as well as for their density, represented by the $d_\text{i/o}$ ratio (Fig. \ref{fig:GVEIIpeople}).
\begin{figure}[t!]
  \centering
  \includegraphics[width=\columnwidth]{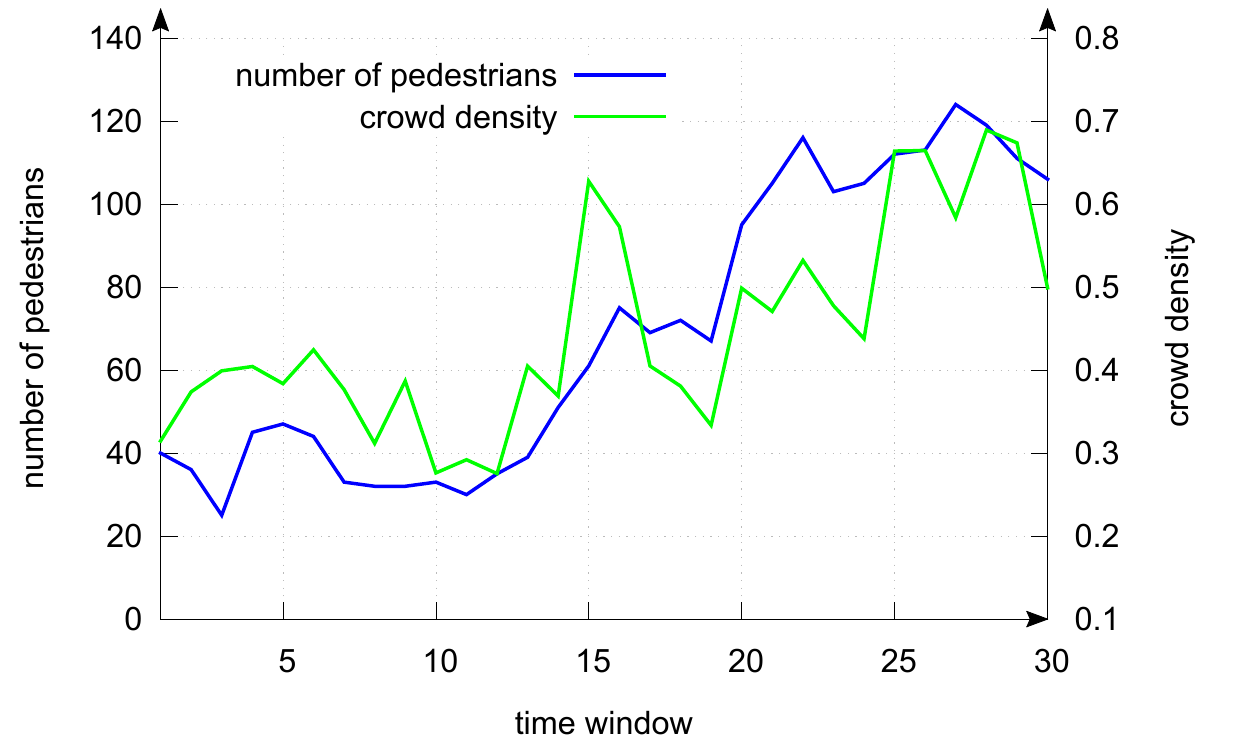}
  \caption{Pedestrians number and $d_\text{i/o}$ ratio temporal evolution in the \texttt{gal1} sequence of \emph{GVEII}.}
  \label{fig:GVEIIpeople}
\end{figure}
In order to underline the importance of capturing changes in density, we compare the batch version of the training algorithm Alg. \ref{BCFW} with a sequential and a fully online version (Fig.~\ref{fig:res_online_comp}). In the former case, examples are fed to the supervised training procedure in temporal order one at a time, while for the latter case, the weights have been initialize to the ones learned batch and the algorithm at each step learns from the previous prediction, thus without supervision.\\
The plot in Fig.~\ref{fig:res_online_comp} shows the performance of the batch training version tends to decrease as the crowd density increases. While the sequential version of the algorithm performs better, it is slow to respond to sudden density changes like in time windows $15$. Indeed, a non-smooth density variation affects negatively the training process, leading to a performance drop further recovered in the subsequent temporal windows. Eventually, this behavior is partially mitigated in the fully online version. The higher performances are motivated by the implicit regularization: using the prediction as training input discourages the learner to drastically modify the weights vector and mimic the smooth variation in crowd density slightly adjusting in time.

\subsection{Performances on Real Detector and Tracker}
\label{exp:real}
Our algorithms assumes the availability of correct trajectories to detect groups, but what happens in a fully automatic video surveillance pipeline where a people detector and tracker are employed?
We carried out experiments
%on the \texttt{student003}, \texttt{eth} and \texttt{hotel} sequences
by extracting pedestrian positions through a state of the art detector~\cite{DollarPAMI14pyramids} and obtaining trajectories by means of a continuous energy minimization method~\cite{Milan:2014:CEM}. We compare with Ge~\emph{et~al.}~\cite{ge_vision-based_2012}, Yamaguchi~\emph{et al.}~\cite{yamaguchi_who_2011} and Shao~\emph{et al.}~\cite{shao14} over the same input data, results are shown in Tab.~\ref{tab:track}. Our proposal outperforms the competitors even in the case of noisy trajectories.

Tracking performances evidence a high number of tracks fragments, namely FRG, that are mainly due to the localization error introduced by the automatic people detector on non-trivial crowded scenes. FRGs are proportional to the number of small new tracks created by the system instead of correctly associating previously tracked objects, with the consequence of splitting ideal tracks into temporally disjoint segments.\\
A high FRG number affects the group detection performance as the $d_{ph}$ and $d_{ca}$ features are computed when the trajectories are simultaneously present in the scene and thus merging temporal disjoint fragments is strongly discouraged by the correlation clustering algorithm.
%To cope with the fragmentation noise returned by the tracker, we study how performances of group detection vary according to the length of the time window on which features are computed.
Intuitively, by reducing the size of the window we are able to minimize the number of split trajectories at each example and recover most of the original performances, as shown in Fig.~\ref{fig:detector_tracker}(c).
The improvement is basically achieved through the joint adoption of socially founded features and structural learning that weights the features according to the observed noisy trajectories. The experiment allow us to conclude that even in the case of a real application and imprecise input data the strengths of the proposed algorithm are maintained because are strongly related to the social rules that govern the group formation process, these rules are not data dependent and hold despite the applied feature extraction techniques.

\begin{figure}[t]
	\centering
	\includegraphics[width=\columnwidth]{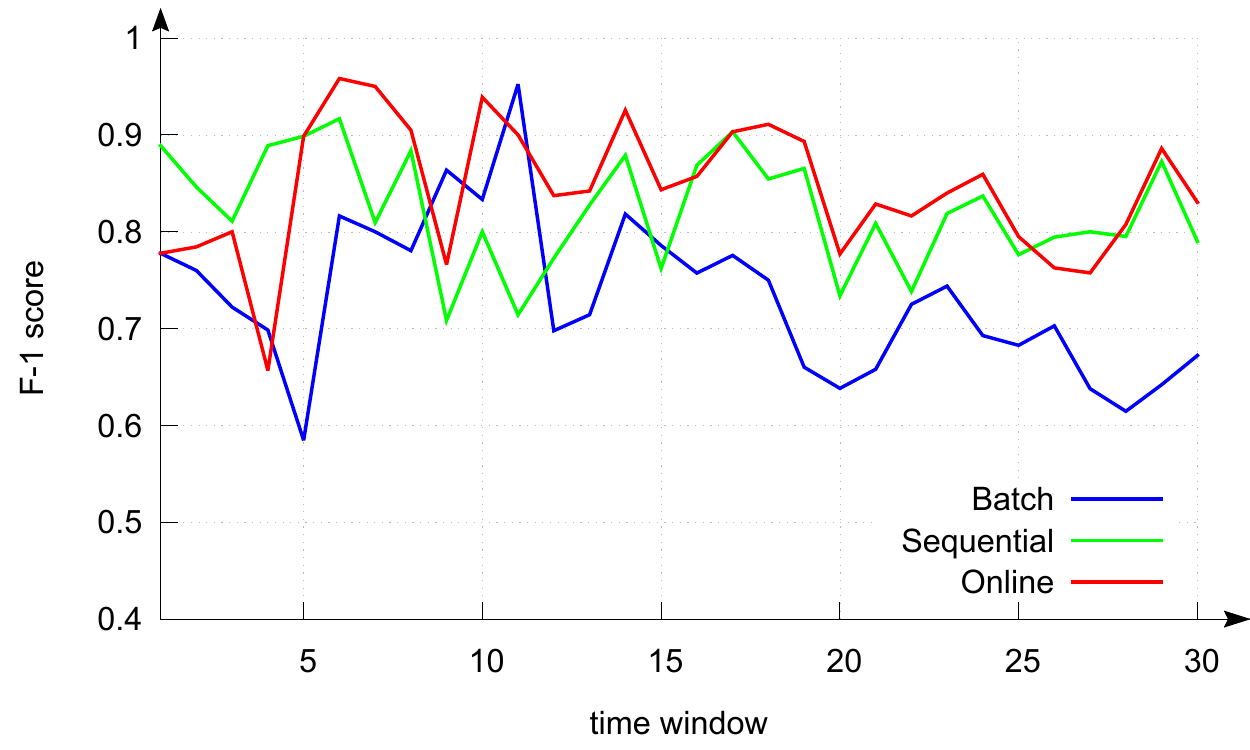}
	\caption{F-1 score comparison between differently trained version of the our method on \texttt{gal1} of \emph{GVEII}.}
	\label{fig:res_online_comp}
\end{figure}

\begin{figure*}[t!]
  \centering
  %\sidesubfloat[]{
  \subfloat[]{
    \includegraphics[width=0.28\textwidth]{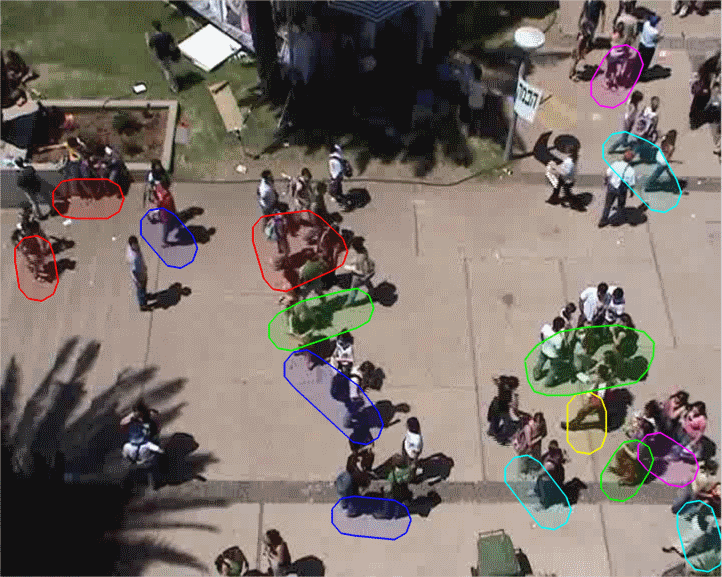}
	}
  %\sidesubfloat[]{
  \subfloat[]{
      \includegraphics[width=0.28\textwidth]{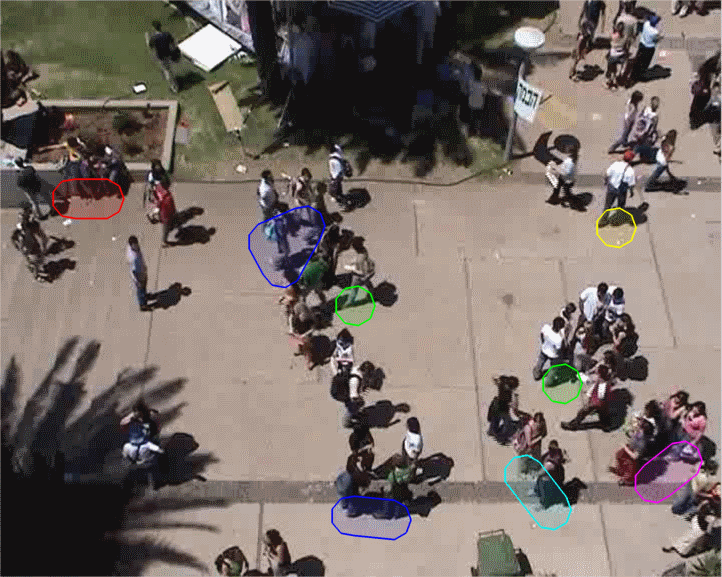}
	}
  \subfloat[]{
      \includegraphics[width=0.4\textwidth]{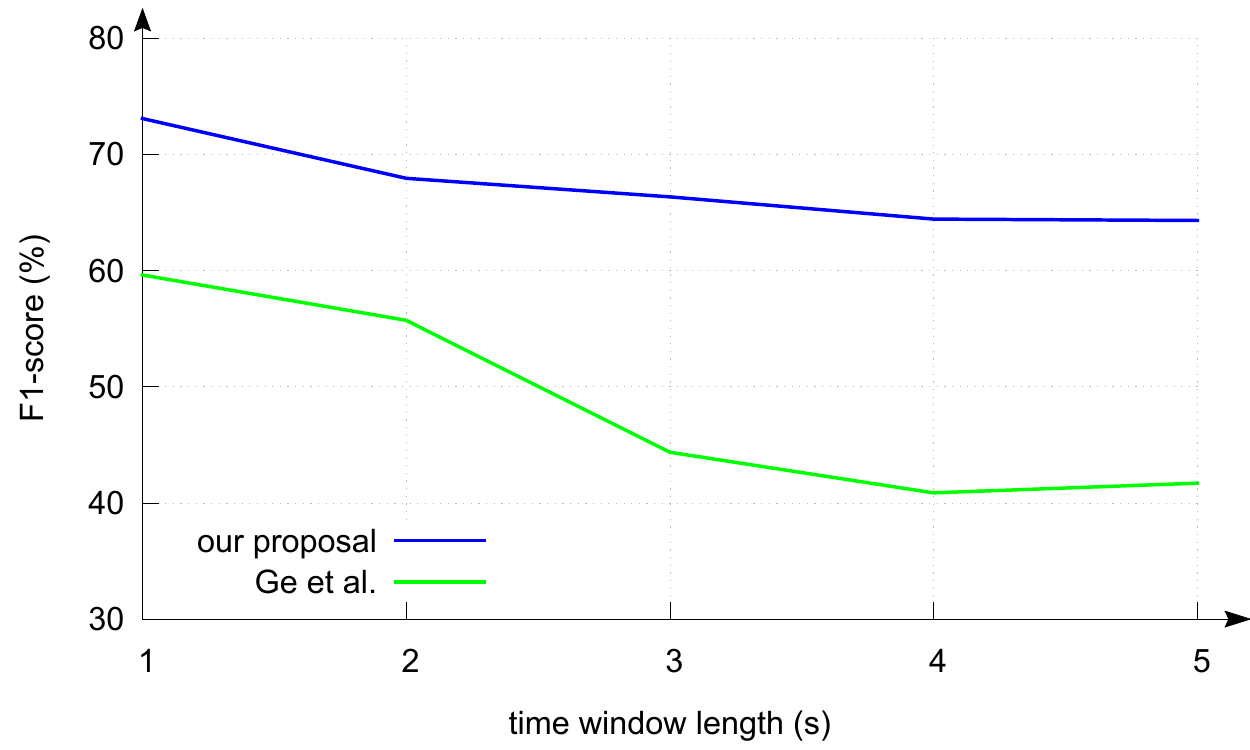}
	}
  \caption{Group detection results on \texttt{student003} are displayed when corrected tracks are used (a) and when input with people detector and tracker automatic responses (b). Regardless of the input noise, most of the groups can still be identified. This is due to the robustness of the features employed during learning and to the decrease in length of the time window (c) which prevents fragmented tracks to be split in different groups.}
  \label{fig:detector_tracker}
\end{figure*}

\begin{figure*}[t!]
  	\centering
  	\subfloat[\texttt{1airport1}]{
    	\includegraphics[width=0.235\textwidth]{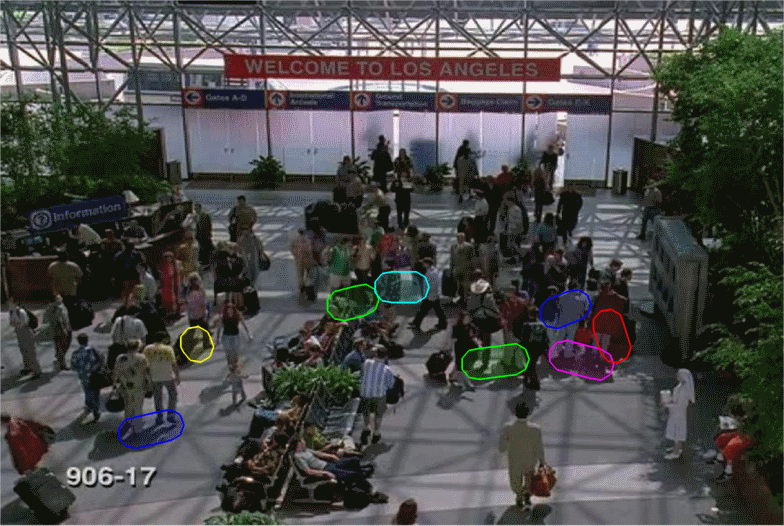}
	}
  	\subfloat[\texttt{1manko3}]{
    	\includegraphics[width=0.235\textwidth]{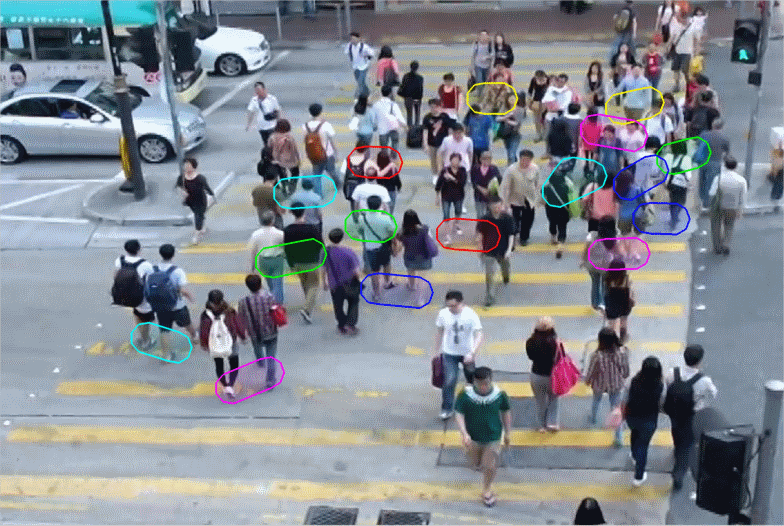}
	}
  	\subfloat[\texttt{2jiansha5}]{
    	\includegraphics[width=0.235\textwidth]{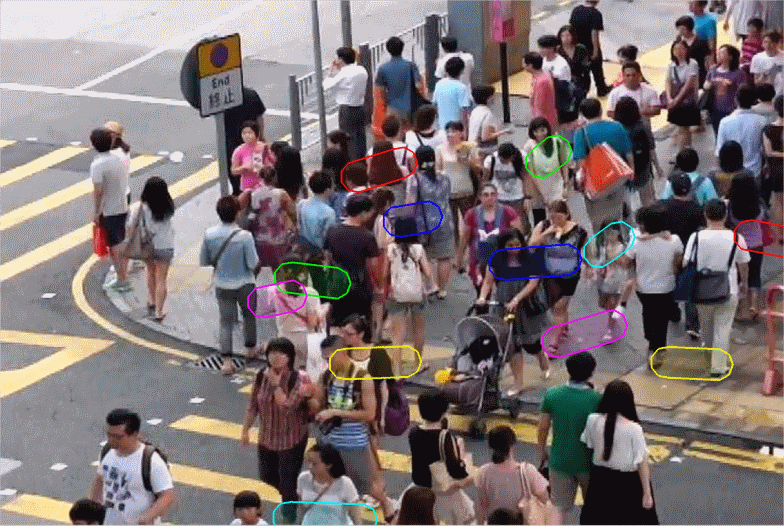}
	}
	\subfloat[\texttt{randomcross3}]{
    	\includegraphics[width=0.235\textwidth]{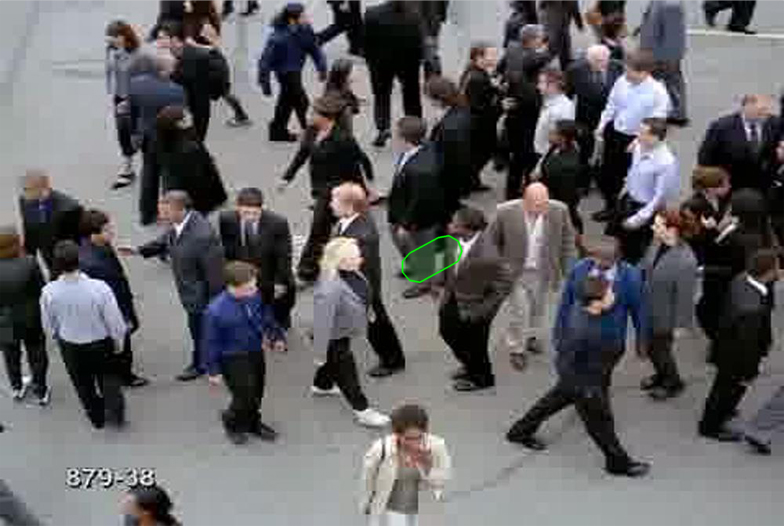}
	}\\
	\subfloat[\texttt{3shatian6}]{
    	\includegraphics[width=0.235\textwidth]{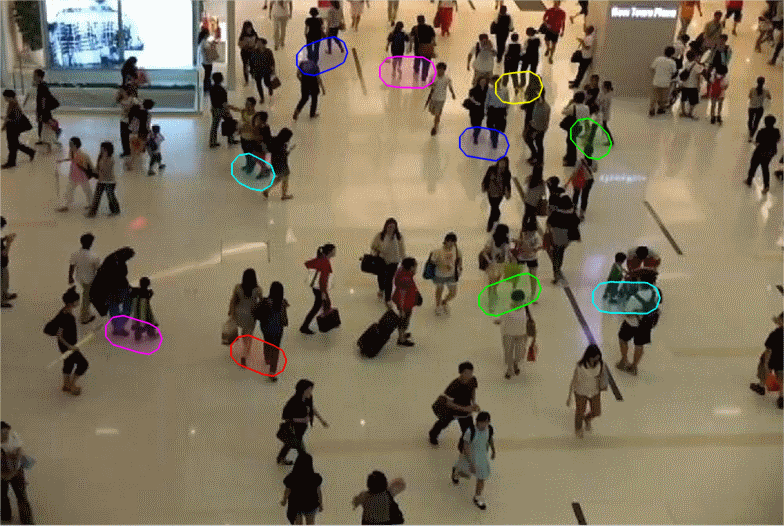}
	}
	\subfloat[\texttt{seq1}]{
    	\includegraphics[width=0.235\textwidth]{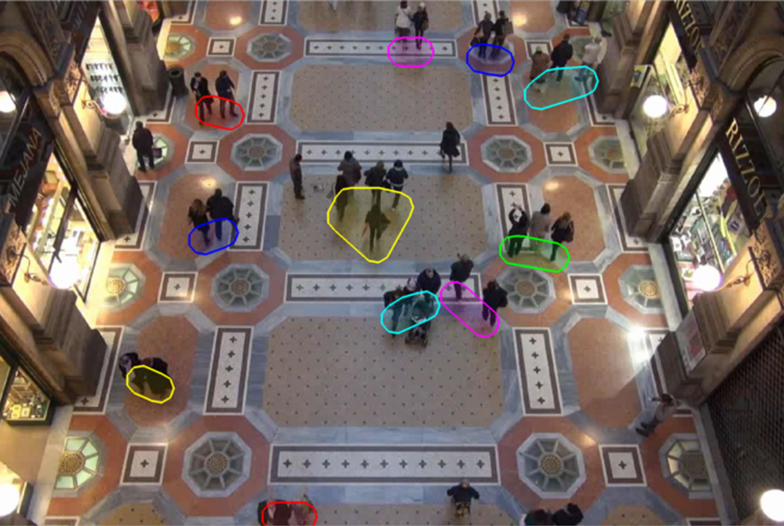}
	}
  	\subfloat[\texttt{eth}]{
      	\includegraphics[width=0.235\textwidth]{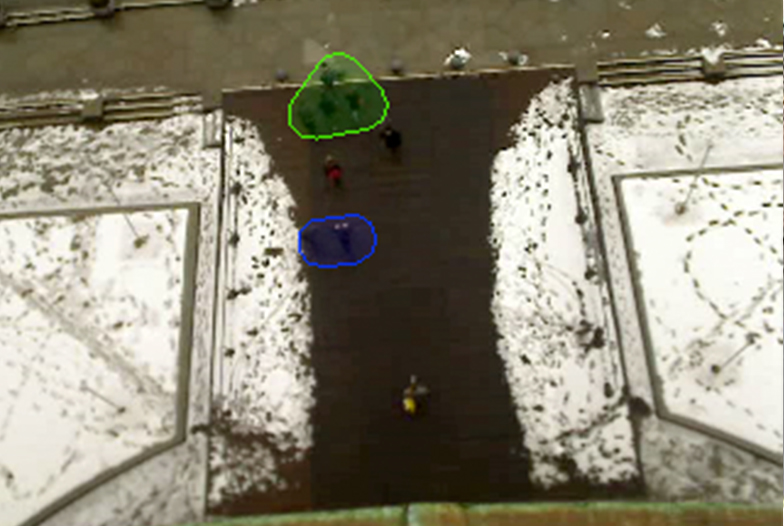}
	}
	\subfloat[\texttt{hotel}]{
      	\includegraphics[width=0.235\textwidth]{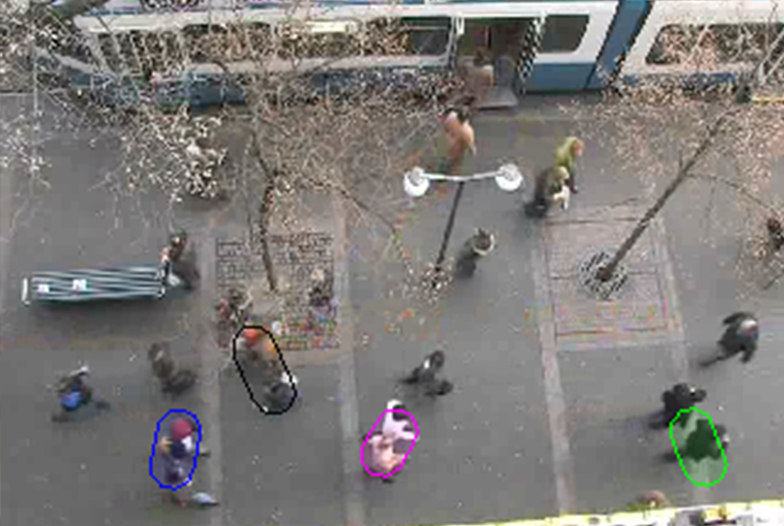}
	}
  \caption{Examples of groups detected through our method: sequences from (a) to (e) are from the \emph{MPT-$20$x$100$}, while (f) is part of \emph{GVEII} and finally, (g) and (h) belong to the \emph{BIWI} dataset. Groups are identified regardless of the scene context and errors are visually acceptable, as in (d).}
  \label{fig:more_results}
\end{figure*}

\section{Conclusion}
In this work, we pointed out the need to approach the task of detecting social groups in crowds from a learning perspective.
Many existing methods rely on specifically tuned parameters that limit their applicability in real world scenarios.
Our intuition is that there are crowds that preserve the same concept of social group, but in many cases this concept cannot be distilled
from spatial consideration only. We thus defined a set of social-inspired and strongly motivated features able to capture and characterize different groups peculiarities.
To learn a socially meaningful clustering rule to group pedestrians, we relied on the Structural SVM framework and designed a peculiar loss function able to account for
singletons as well as for group errors.
Even though the algorithm was originally designed to work with exact trajectories, we replicated the experiments on noisy tracklets extracted by a detector/tracker obtaining state-of-the-art results.
Moreover, we proposed an online training version of the method, able to achieve superior generalization performances on crowds with variable density.

We did note, however, that as we consider wider portions of the scene, the chance that many different densities groups coexist in different locations increases, leading to the necessity to learn more than one clustering rule per scene. To resolve this problem we plan, as future work, to learn a set of different distance measures and use latent variables to choose the most appropriate given a particular zone. Code and datasets are made publicly available\footnote{\texttt{http://imagelab.ing.unimore.it/group-detection}} in order to reproduce this paper results and allow the community to improve the proposed method.

% if have a single appendix:
%\appendix[Proof of the Zonklar Equations]
% or
%\appendix  % for no appendix heading
% do not use \section anymore after \appendix, only \section*
% is possibly needed

% use appendices with more than one appendix
% then use \section to start each appendix
% you must declare a \section before using any
% \subsection or using \label (\appendices by itself
% starts a section numbered zero.)
%%
%
%
%\appendices
%\section{Proof of the First Zonklar Equation}
%Appendix one text goes here.
%
%% you can choose not to have a title for an appendix
%% if you want by leaving the argument blank
%\section{}
%Appendix two text goes here.
%
%
%% use section* for acknowledgement
%\section*{Acknowledgment}
%
%
%The authors would like to thank...

% Can use something like this to put references on a page
% by themselves when using endfloat and the captionsoff option.
\ifCLASSOPTIONcaptionsoff
  \newpage
\fi

% trigger a \newpage just before the given reference
% number - used to balance the columns on the last page
% adjust value as needed - may need to be readjusted if
% the document is modified later
%\IEEEtriggeratref{8}
% The "triggered" command can be changed if desired:
%\IEEEtriggercmd{\enlargethispage{-5in}}

% references section

% can use a bibliography generated by BibTeX as a .bbl file
% BibTeX documentation can be easily obtained at:
% http://www.ctan.org/tex-archive/biblio/bibtex/contrib/doc/
% The IEEEtran BibTeX style support page is at:
% http://www.michaelshell.org/tex/ieeetran/bibtex/
\bibliographystyle{IEEEtran}
% argument is your BibTeX string definitions and bibliography database(s)
\bibliography{IEEEabrv,main}

%
% <OR> manually copy in the resultant .bbl file
% set second argument of \begin to the number of references
% (used to reserve space for the reference number labels box)
%\begin{thebibliography}{1}

%\bibitem{IEEEhowto:kopka}
%H.~Kopka and P.~W. Daly, \emph{A Guide to \LaTeX}, 3rd~ed.\hskip 1em plus
%  0.5em minus 0.4em\relax Harlow, England: Addison-Wesley, 1999.
%\end{thebibliography}

% biography section
% 
% If you have an EPS/PDF photo (graphicx package needed) extra braces are
% needed around the contents of the optional argument to biography to prevent
% the LaTeX parser from getting confused when it sees the complicated
% \includegraphics command within an optional argument. (You could create
% your own custom macro containing the \includegraphics command to make things
% simpler here.)
%\begin{IEEEbiography}[{\includegraphics[width=1in,height=1.25in,clip,keepaspectratio]{mshell}}]{Michael Shell}
% or if you just want to reserve a space for a photo:

\begin{IEEEbiography}[{\includegraphics[width=1in,height=1.25in,clip]{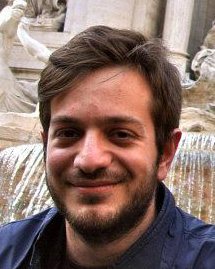}}]{Francesco Solera}
obtained a master degree in computer engineering from the University of Modena and Reggio Emilia in 2013. He is now a PhD candidate within the ImageLab group in Modena, researching on applied machine learning and social computer vision.
\end{IEEEbiography}

% if you will not have a photo at all:
\begin{IEEEbiography}[{\includegraphics[width=1in,height=1.25in,clip]{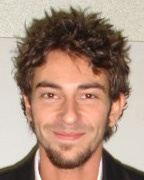}}]{Simone Calderara}
received a computer engineering master degree in 2004 and a PhD degree in 2009 from the University
of Modena and Reggio Emilia, where he is now an assistant professor within the Imagelab group. His current research interests include computer vision and machine learning applied to human
behavior analysis, visual tracking in crowded scenarios and time series analysis for forensic applications.
\end{IEEEbiography}

\begin{IEEEbiography}[{\includegraphics[width=1in,height=1.25in,clip]{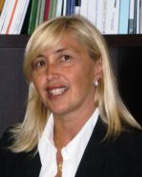}}]{Rita Cucchiara}
received her master degree in electronic engineering and the PhD degree
in computer engineering from the University of Bologna, Italy, in 1989 and 1992 respectively.
Since 2005, she is a full professor at University of Modena and Reggio Emilia, Italy, where she heads the
ImageLab group and the SOFTECH-ICT research center. Her research focuses on pattern
recognition, computer vision and multimedia.
\end{IEEEbiography}

% You can push biographies down or up by placing
% a \vfill before or after them. The appropriate
% use of \vfill depends on what kind of text is
% on the last page and whether or not the columns
% are being equalized.

%\vfill

% Can be used to pull up biographies so that the bottom of the last one
% is flush with the other column.
%\enlargethispage{-5in}

% that's all folks
\end{document}